\documentclass[10pt,journal,compsoc]{IEEEtran}

\usepackage{mathrsfs}
\usepackage{hyperref}
\usepackage{courier}
\usepackage{epsfig}
\usepackage{graphicx}
\usepackage{amsmath}
\usepackage{amsthm}
\usepackage{mathtools}
\usepackage{amssymb}
\usepackage{verbatim}
\usepackage{booktabs}
\usepackage{algpseudocode}
\usepackage{amscd}
\usepackage{times}
\usepackage{subcaption}
\usepackage{lipsum}
\usepackage{algorithm}
\usepackage[table]{xcolor}
\usepackage{color}
\usepackage{enumerate}
\usepackage{tikz}
\usepackage{multirow}
\usepackage{bm}
\usepackage{diagbox}
\usepackage{pgfplots}
\usepackage{tikz}
\newtheorem*{theorem}{Theorem}
\newtheorem{thm}{Theorem}

\newtheorem{lem}{Lemma}
\newtheorem{prp}{Proposition}
\newtheorem{dfn}{Definition}
\newtheorem{prb}{Problem}

\newtheorem{ex}{Example}

\def\L{\mathcal{L}}

\def\N{\mathcal{N}}

\def\X{\mathcal{X}}
\def\Y{\mathcal{Y}}

\def\cS{\mathcal{S}}

\def\I{\mathcal{I}}

\def\J{\mathcal{J}}
\def\P{\mathcal{P}}

\def\Re{\mathbb{R}}
\def\Ce{\mathbb{C}}
\def\R{\mathcal{R}}

\def\V{\mathcal{V}}

\def\W{\mathcal{W}}

\def\0{\boldsymbol{0}}
\def\1{\boldsymbol{1}}

\DeclareMathOperator*{\Span}{Span}

\DeclareMathOperator*{\argmin}{argmin}

\DeclareMathOperator*{\rank}{rank}

\DeclareMathOperator*{\In}{in}

\newcommand{\myparagraph}[1]{\smallskip\noindent\textbf{#1.}}

\newcommand{\ra}[1]{\renewcommand{\arraystretch}{#1}}

\begin{document} 

\graphicspath{{figures/}}

\title{An algebraic-geometric approach for linear regression without correspondences}

\author{
    \IEEEauthorblockN{Manolis C. Tsakiris\IEEEauthorrefmark{1}, Liangzu Peng\IEEEauthorrefmark{1},  Aldo Conca\IEEEauthorrefmark{2}, Laurent Kneip\IEEEauthorrefmark{1}, \\ Yuanming Shi\IEEEauthorrefmark{1} and Hayoung Choi \IEEEauthorrefmark{1}}\\ 
    \vspace{0.2in}
    \IEEEauthorblockA{\IEEEauthorrefmark{1} \small School of Information Science and Technology, ShanghaiTech University, Shanghai, China.}    \\
    \vspace{0.05in} 
    \IEEEauthorblockA{\IEEEauthorrefmark{2} \small Department of Mathematics, University of Genova, Genova, Italy.} \\         
    \IEEEcompsocitemizethanks{\IEEEcompsocthanksitem Correspondence to M. C. Tsakiris at mtsakiris@shanghaitech.edu.cn}
} \normalsize

\IEEEtitleabstractindextext{
\begin{abstract}
Linear regression without correspondences is the problem of performing a linear regression fit to a dataset for which the correspondences between the independent samples and the observations are unknown. Such a problem naturally arises in diverse domains such as computer vision, data mining, communications and biology. In its simplest form, it is tantamount to solving a linear system of equations, for which the entries of the right hand side vector have been permuted. This type of data corruption renders the linear regression task considerably harder, even in the absence of other corruptions, such as noise, outliers or missing entries. Existing methods are either applicable only to noiseless data or they are very sensitive to initialization or they work only for partially shuffled data. In this paper we address these issues via an algebraic geometric approach, which uses symmetric polynomials to extract permutation-invariant constraints that the parameters $\xi^* \in \Re^n$ of the linear regression model must satisfy. This naturally leads to a polynomial system of $n$ equations in $n$ unknowns, which contains $\xi^*$ in its root locus. Using the machinery of algebraic geometry we prove that as long as the independent samples are generic, this polynomial system is always consistent with at most $n!$ complex roots, regardless of any type of corruption inflicted on the observations. The algorithmic implication of this fact is that one can always solve this polynomial system and use its most suitable root as initialization to the Expectation Maximization algorithm. To the best of our knowledge, the resulting method is the first working solution for small values of $n$ able to handle thousands of fully shuffled noisy observations in milliseconds. 
\end{abstract}
\begin{IEEEkeywords}
linear regression without correspondences, linear regression with shuffled data, shuffled linear regression, unlabeled sensing, homomorphic sensing, expectation maximization, algebraic geometry
\end{IEEEkeywords}}

\maketitle

%%%%%%%%%%%%%%%%%%%%%%%%%%%%%%%%%%%%%%%%%%%%%%%%%%
\section{Introduction}\label{section:Introduction}

In the span of more than $200$ years since the work of Legendre \cite{Legendre:1805} and Gauss \cite{Gauss:1809}, linear regression has grown to be a cornerstone of statistics, with applications in almost every branch of science and engineering that involves computing with data. In its simplest form, classical linear regression considers a data model whose output $y$ is a linear function of known functions of the input data $u$. More precisely, given correspondences $\{u_j,y_j\}_{j=1}^m$ with $u_j \in \Re^s, \, y_j \in \Re$ and known functions $a_i:\Re^s \rightarrow \Re$, one seeks to find real numbers $\xi = [\xi_1,\dots,\xi_n]^\top \in \Re^n$ such that\footnote{A classical example is when $y$ is a polynomial function of $u$, in which case the $a_i$ represent the monomials that appear in the polynomial and $\xi_i$ are their coefficients to be fitted.}
\begin{align}
    y_j &\approx a_j^\top \xi, \, \, \, \forall j=1,\dots,m,  \\
    a_j &:= [a_1(u_j),\dots,a_n(u_j)]^\top. \label{eq:FunctionsOfInput}
\end{align} Between the least-squares solution of Gauss and modern approaches designed to deal with highly corrupted data \cite{Maronna:Technometrics2005,Wang:CVPR2015,Lerman:arXiv18,Tsakiris:JMLR18}, a literature too vast to enumerate has been devoted to progressively more complicated versions of the linear regression problem \cite{Marona:Wiley2006}.

%%%%%
\subsection{Linear Regression Without Correspondences} In this paper we are interested in a particular type of data corruption, which is \emph{lack of correspondences}. In such a case, one is given the input samples, or more precisely functions of the input samples as in \eqref{eq:FunctionsOfInput}, i.e.,
\begin{align}
    A = [a_1, \dots, a_m]^\top \in \Re^{m \times n},
\end{align} and a \emph{shuffled} version $y = [y_{\pi(1)},\dots,y_{\pi(m)}]^\top \in \Re^m$, of the observations $y_1,\dots,y_m$, where $\pi$ is an unknown permutation of $[m]=\{1,\dots,m\}$. This problem of \emph{linear regression without correspondences} \cite{Hsu:NIPS17} also known as \emph{linear regression with shuffled data} \cite{Pananjady:TIT18}, \emph{shuffled linear regression} \cite{Abid:arXiv17,Abid:arXiv18}, \emph{unlabeled sensing} \cite{Unnikrishnan:TIT18} or \emph{permuted linear model} \cite{Choi:ISIT18}, can be stated in the absence of any other data corruptions as follows:
\begin{prb} \label{prb:SLR}
Suppose we are given a matrix $A \in \Re^{m \times n}$ with $m > n$, and a vector $y =(\Pi^*)^\top A \xi^* \in \Re^m$, where $\xi^* \in \Re^n$ is some vector and $\Pi^*$ an $m \times m$ permutation matrix. We wish to efficiently compute $\xi^*$ when $\Pi^*$ is unknown and $m \gg n$. In other words, without knowing $\Pi^*$, we want to solve the linear system 
\begin{align}
\Pi^* y = A x \label{eq:SLR}.
\end{align} 
\end{prb} 
 
\noindent Problem \ref{prb:SLR} arises in a wide variety of applications, such as 
1) computer vision, e.g., multi-target tracking \cite{Ji:IEEEAccess2019} %\cite{POORE20061074}, 
and pose/correspondence estimation \cite{David2004,5459318}, 2) record linkage \cite{Lahiri:JASA2005,Slawski:EJS19,Slawski:arXiv19}, 3) biology, e.g., for cell tracking  \cite{2012arXiv1208.1070R}, genome-assembly \cite{DNA1999}, and identical tokens in signaling pathways \cite{6875147}, 4) communication networks, e.g., for data de-anonymization \cite{4531148,5062142}, and low-latency communications in Internet-Of-Things networks \cite{Choi:ISIT18},  5) signal processing, e.g., when dealing with signals sampled in the presence of time jitter \cite{1057717}, or in applications where a spatial field is being inaccurately sampled by a moving sensor \cite{6463464,6376190}.

%%%%%%%%%%%
\subsection{Prior Art} \label{subsection:PriorArt}

Over the past years there has been a considerable amount of work on instances of Problem \ref{prb:SLR} that come with additional structure in diverse contexts, e.g., see the excellent literature reviews in \cite{Pananjady:TIT18,Unnikrishnan:TIT18,Slawski:EJS19}. Nevertheless, it has only been until very recently that the problem of shuffled linear regression has been considered in its full generality. In fact, the main achievements so far have been concentrating on a theoretical understanding of the conditions that allow unique recovery of $\xi^*$ or $\Pi^*$; see \cite{Hsu:NIPS17,Abid:arXiv17,Pananjady:TIT18,Unnikrishnan:TIT18,Abid:arXiv18,Choi:ISIT18,8006567,6034717,Haghighatshoar:TSP18,Slawski:EJS19,Dokmanic:SPL-19,Tsakiris:ICML19,Tsakiris:ECHS-arXiv19,Slawski:arXiv19}. 

Letting $A$ be drawn at random from any continuous probability distribution, \cite{Unnikrishnan:TIT18} proved that any such $\xi^*$ can be uniquely recovered with probability $1$ as long as\footnote{While the present paper was under review, this result was greatly generalized in \cite{Tsakiris:ICML19,Tsakiris:ECHS-arXiv19,Dokmanic:SPL-19}.} $m \ge 2n$. If on the other hand $m<2 n$, then $\xi^*$ is not unique with probability $1$. Further considering additive random noise on $y$, the authors in \cite{Hsu:NIPS17} established lower bounds on the SNR, below which for any estimator there is a $\xi^*$ whose estimation error is \emph{large}. With $A$ drawn from a normal distribution, $\xi^*$ unknown but fixed and $y$ corrupted by additive random noise, \cite{Pananjady:TIT18} showed that, as long as the SNR exceeds a threshold, $\Pi^*$ coincides with high probability with the \emph{Maximum Likelihood Estimator} (MLE), which they defined as (also considered in \cite{Abid:arXiv17,Abid:arXiv18})
\vspace{-0.04in}
\begin{align}
(\widehat{\Pi}_{\text{ML}}, \widehat{x}_{\text{ML}}) = \argmin_{\Pi, x} \left\| \Pi y - A x \right\|_2 \label{eq:MLE},
\end{align} where $\Pi$ in \eqref{eq:MLE} is constrained to be a permutation matrix. If on the other hand the SNR is not large enough, $\widehat{\Pi}_{\text{ML}}$ differs from $\Pi^*$ with high probability, in agreement with the results of  \cite{Hsu:NIPS17}. This was further complemented by \cite{Unnikrishnan:TIT18}, which showed that $\widehat{x}_{\text{ML}}$ is locally stable under noise, in the sense that as the SNR tends to infinity $\widehat{x}_{\text{ML}}$ tends to $\xi^*$. However, according to \cite{Abid:arXiv17}, for SNR fixed, $\widehat{x}_{\text{ML}}$ is asymptotically inconsistent. Interestingly, if the data are only sparsely shuffled, the work of \cite{Slawski:EJS19} proved that $\xi^*$ coincides with the optimal solution of a robust $\ell_1$ regression problem.

On the algorithmic front of solving Problem \ref{prb:SLR} much less has been achieved. In \cite{Unnikrishnan:TIT18} the authors write, \emph{...although we showed that recovery of the unknown $\xi^*$ is possible from unlabeled measurements, we do not study the problem of designing an efficient algorithm to recover $\xi^*$. Our solution is to consider all possible permutations of the unlabeled observations which might be prohibitively complex in large dimensional problems}. Indeed, this involves checking whether the linear system $\Pi y = A x$ is consistent for each permutation $\Pi$ among the $m!$ permutations of the $m$ entries of $y$, yielding a complexity $\mathcal{O}((m)! m n^2)$. A more efficient algorithm is that of \cite{Hsu:NIPS17}, which is able to reduce the complexity as a function of $m$ to a factor of at least $m^7$. However, as the authors of \cite{Hsu:NIPS17} write, their algorithm \emph{strongly exploits the assumption of noiseless measurements} and \emph{is also very brittle and very likely fails in the presence of noise}; the same is true for the $\mathcal{O}(m^n)$ complexity algorithm of \cite{7953021}. Finally, the authors in \cite{Hsu:NIPS17} write \emph{we are not aware of previous algorithms for the average-case problem in general dimension $n$}. In the same paper a $(1+\epsilon)$ approximation algorithm with theoretical guarantees and of complexity $\mathcal{O}((m/\epsilon)^n)$ is proposed, which however, \emph{is not meant for practical deployment, but instead is intended to shed light on the computational difficulty of the least squares problem} \eqref{eq:MLE}. Indeed, as per \cite{Pananjady:TIT18} for\footnote{The case $n=1$ is well understood and solved at a complexity of $\mathcal{O}(m \log(m))$ by sorting (\cite{Pananjady:TIT18}, \cite{Abid:arXiv17}).} $n>1$  this is an NP-hard problem\footnote{While the present paper was under review, \cite{Tsakiris:ICML19} proposed an empirical algorithm based on branch \& bound and dynamic programming.}. On the other hand, the approach that seems to be the predominant one in terms of practical deployment is that of solving \eqref{eq:MLE} via alternating minimization \cite{Abid:arXiv18}: given an estimate for $\xi^*$ one computes an estimate for $\Pi^*$ via sorting; given an estimate for $\Pi^*$ one computes an estimate for $\xi^*$ via ordinary least-squares. However, this approach is very sensitive to initialization and generally works only for partially shuffled data; see \cite{Abid:arXiv18} for a \emph{soft} variation of this alternating scheme. In conclusion and to the best of our knowledge, there does not yet exist an algorithm for solving Problem \ref{prb:SLR} that is i) theoretically justifiable, ii) efficient and iii) robust to even mild levels of noise.

%%%%
\subsection{Contributions}

In this work, we contribute to the study of Problem \ref{prb:SLR} on both theory (\S \ref{section:TheoreticalContributions}) and algorithms (\S \ref{section:AlgorithmicImplications}). On the theoretical level, we show that for generic noiseless data $y,A$, there is a unique solution $\Pi^*,\xi^*$, as soon as $m>n$. We show that $\xi^*$ is contained in the root locus of a system of $n$ polynomial equations in $n$ unknowns. Using tools from algebraic geometry, we show that this polynomial system is always consistent with at most $n!$ complex solutions, regardless of any noise that may further corrupt the observations $y$. Furthermore, we show that the euclidean distance of $\xi^*$ from the root locus of the noisy system is bounded by a polynomial function of the noise which vanishes for zero noise. The algorithmic implication of these results is that (under the genericity assumption) one can always solve the noisy polynomial system and use a simple criterion to identify its most appropriate root to be used as initialization for computing the MLE estimator \eqref{eq:MLE} via alternating minimization. Even though solving the polynomial system entails in principle exponential complexity in $n$, its complexity in $m$ is linear for any $n$. Furthermore, we use methods from automatic algebraic-geometric solver generation to obtain highly efficient solvers for $n=3,4$. Deriving the SNR rate of the resulting estimator is a challenging problem that we do not pursue in this paper. Even so, we provide empirical evidence according to which 
for $n \le 5$ our approach is the first working solution to linear regression without correspondences that remains stable under noise and has manageable complexity. As an example, for $n=4$, $m=10,000$, and $1\%$ additive noise, our method computes in $313$ milliseconds a solution that is within $0.6\%$ error from the ground truth.

%%%%%%%%%%%
\section{Theoretical Contributions} \label{section:TheoreticalContributions}

The main contribution of this paper is to develop the theory to an algebraic geometric approach for solving Problem \ref{prb:SLR}. The key idea, described in detail in \S \ref{subsection:SymmetricPolynomials}, uses symmetric power-sum polynomials to eliminate the unknown permutation, thus resulting in a polynomial system $\P$ of $n$ equations in $n$ unknowns. These polynomials were considered implicitly in the statistical approach of \cite{Abid:arXiv17}, towards constructing a \emph{self-moment} estimator. The authors of that paper wrote, \emph{...in fact there may not be a solution to the system}, which led them to compute their estimator via gradient descent on a highly non-convex objective function, a procedure lacking theoretical guarantees and very sensitive to initialization. The geometric significance of the polynomial system $\P$ was also recognized by the last two authors of the present paper in the short conference paper \cite{Choi:ISIT18}, but important questions such as
\begin{enumerate}
\item ``does $\P$ have finitely many solutions?" or
\item ``does $\P$ have any solutions in the presence of noise?", 
\end{enumerate} were left as open problems.

In this paper we answer these two questions in the affirmative in \S \ref{subsection:MainResults}, via Theorems \ref{thm:Main} and \ref{thm:Perturbations}-\ref{thm:Noise} respectively. The main message is that if the input data $A \in \Re^{m \times n} $ are generic (to be made precise in \S \ref{subsection:Genericity}), then using $n$ power-sum polynomials of degrees $1,2,\dots,n$ as constraints, defines an algebraic variety that consists of at most $n!$ points, among which lies the solution $\xi^*$ of the shuffled linear regression Problem \ref{prb:SLR}. In addition, the same conclusion holds true in the case where the observation vector $y$ has undergone any type of corruption: the variety defined by the \emph{noisy} polynomials is non-empty and consisting of at most $n!$ complex points. This guarantees that the equations are almost always consistent even if the data are imperfect, which enables algorithmic development as discussed in \S \ref{section:AlgorithmicImplications}. Even though we do provide a bound on the effect of the noise on the solutions of the algebraic equations $\P$, deriving the SNR rate for the resulting estimator involves further challenges. Hence we leave such a statistical analysis as an open problem, which we hope that the present algebraic developments will facilitate in solving. 

The proofs of Theorems \ref{thm:Main}-\ref{thm:Noise}, given in \S \ref{subsection:Proofs}, require a thorough understanding of the notion of dimension of polynomial systems of equations. We describe the necessary notions in the series of self-contained appendices \ref{appendix:geometric-dimension}-\ref{appendix:initial-ideals} in an expository style for the benefit of the reader who is not familiar with algebraic geometry.

%%%%
\subsection{Genericity and Well-Posedness} \label{subsection:Genericity}

Before we are in a position to state our main results, i.e., Theorems \ref{thm:Main}-\ref{thm:Noise} described in \S \ref{subsection:MainResults}, we need to clarify what we mean when we refer to $A,y$ as being generic (\S \ref{subsubsection:well-posedness}), and also settle the well-posedness of Problem \ref{prb:SLR} (\S \ref{subsubsection:Uniqueness}).

%%%
\subsubsection{The notion of generic \texorpdfstring{$A,y$}{Lg}} \label{subsubsection:well-posedness}

We start with an example. 

\begin{ex} \label{ex:Generic}
Consider the $2 \times 2$ matrix 
\begin{align}
B = 
\begin{bmatrix}
b_{11} & b_{12} \\
b_{21} & b_{22}
\end{bmatrix}, 
\end{align} where the entries of $B$ are real numbers. Then $B$ is invertible if and only if 
$b_{11} b_{22} - b_{12} b_{21} \neq 0$. Now consider the polynomial ring $\Re[x_{11},x_{12},x_{21},x_{22}]$ in four variables, with each of them corresponding to an entry of $B$. The equation 
\begin{align}
x_{11} x_{22} - x_{12} x_{21}=0
\end{align} defines a hypersurface
$\V(x_{11} x_{22} - x_{12} x_{21})$ of $ \Re^4 \cong \Re^{2 \times 2}$, and $B \in \V(x_{11} x_{22} - x_{12} x_{21})$ if and only if $B$ is non-invertible. This hypersurface has measure zero, say, under the standard normal distribution of $\Re^4$, and hence if one samples $B$ at random from this distribution, $B$ will lie outside of $\V(x_{11} x_{22} - x_{12} x_{21})$ with probability $1$. We express this by saying that ``if $B$ is a generic $2 \times 2$ matrix, then $B$ is invertible".  \end{ex}

As Example \ref{ex:Generic} suggests, we usually attach the attribute \emph{generic} to an object $\mathscr{O}$ (matrix $B$ in Example \ref{ex:Generic}) with respect to some property $\mathscr{P}$ of $\mathscr{O}$ (invertibility of $B$ in Example \ref{ex:Generic}). We say that ``if $\mathscr{O}$ is generic then $\mathscr{P}$ is true", and mean that the set of objects for which $\mathscr{P}$ is not true forms a zero-measure set of the underlying space that parametrizes that object under some continuous probability distribution. Hence sampling $\mathscr{O}$ at random from that probability distribution results in $\mathscr{O}$ having the property $\mathscr{P}$ with probability $1$. Finally, if there are finitely many properties $\mathscr{P}_1,\dots,\mathscr{P}_t$ of interest with regard to the object $\mathscr{O}$, and if $\mathscr{O}$ is generic with respect to each of the $\mathscr{P}_i$, then $\mathscr{O}$ is generic with respect to all of them; this follows from the fact that the union of finitely many zero-measure sets is a zero-measure set. The connection between algebraic geometry and measure theory that the reader should keep in mind for our purpose here, is that algebraic subvarieties of irreducible varieties (see Definition \ref{dfn:DimensionGeometric}) have zero measure and that the union of finitely many algebraic subvarieties  is an algebraic subvariety.

%%%
\subsubsection{Uniqueness of \texorpdfstring{$\Pi^*, \xi^*$}{Lg}} \label{subsubsection:Uniqueness}
Our first result, Theorem \ref{thm:Uniqueness}, guarantees that Problem \ref{prb:SLR} is well-posed for generic data, in which case it makes sense to talk about a unique permutation $\Pi^*$ and a unique solution $\xi^*$.

\begin{thm} \label{thm:Uniqueness}
Suppose that $m>n$. Then as long as $A \in \Re^{m \times n}$ is a generic matrix and $y$ the permutation of a generic vector in $\R(A)$, $\Pi^*$ and $\xi^*$ in Problem \ref{prb:SLR} are unique.
\end{thm}

\noindent It is interesting to compare Theorem \ref{thm:Uniqueness} with a simplified version of the main result of Unnikrishnan-Haghighatshoar-Vetterli (\cite{Unnikrishnan:TIT18}), stated next for convenience.

\begin{theorem}[U-H-V]
Suppose that $m\ge 2n$. Then as long as $A \in \Re^{m \times n}$ is a generic matrix and $y$ the permutation of any vector in $\R(A)$, $\xi^*$ in Problem \ref{prb:SLR} is unique.
\end{theorem}

Both Theorems \ref{thm:Uniqueness} and U-H-V are concerned with a generic matrix $A$ and the permutation $y$ of some vector $v$ in $\R(A)$. In Theorem \ref{thm:Uniqueness} the vector $v$ is taken to be a generic vector in $\R(A)$, and as it turns out it is enough that $m>n$ for $v$ to be uniquely defined from the data $A,y$. On the other hand, in Theorem U-H-V the vector $v$ is allowed to be \emph{any} vector in $\R(A)$. This is a considerably more difficult setting, and the remarkable proof of \cite{Unnikrishnan:TIT18} reveals that $v$ is still uniquely defined from the data, as long as now\footnote{See \cite{Tsakiris:ICML19,Tsakiris:ECHS-arXiv19} for a generalization of this result to arbitrary linear transformations instead of just permutations.} $m \ge 2n$. 

Finally, we note that in the setting of Theorem U-H-V unique recovery of the permutation $\Pi^*$ is in principle not possible, as per Theorem $10$ in \cite{Unnikrishnan:TIT18}. Instead, one has to either allow for the observed vector $y$ to be generic (i.e., the permutation of a generic vector of $\R(A)$) in which case $\Pi^*$ is uniquely recoverable by Theorem \ref{thm:Uniqueness}, or consider unique recovery with high probability, which is indeed possible even when $y$ is corrupted by noise \cite{Pananjady:TIT18}.

%%%
\subsection{Eliminating \texorpdfstring{$\Pi^*$}{Lg} via Symmetric Polynomials} \label{subsection:SymmetricPolynomials}
To describe the main idea of the algebraic-geometric approach to solving Problem \ref{prb:SLR}, let $\Re[z]:=\Re[z_1,\dots,z_m]$ be the ring of polynomials with real coefficients over variables $z:=[z_1,\dots,z_m]^\top$. A polynomial\footnote{We do not distinguish between $p$ and $p(z)$.} $p \in \Re[z]$ is called \emph{symmetric} if it is invariant to any permutation of the variables $z$, i.e.,  
\small
\begin{align}
p(z):=p(z_1,\dots,z_m) = p(z_{\pi(1)},\dots,z_{\pi(m)})=:p(\Pi z),
\end{align} 
\normalsize 
where $\pi$ is a permutation on $\{1,\dots,m\}=:[m]$ (i.e., a bijective function from $[m]$ to itself) 
and $\Pi$ is an $m \times m$ matrix representing this permutation, with $i$th row the canonical vector $e_{\pi(i)}^\top$ of all zeros, except a $1$ at position $\pi(i)$. 

Now let $A, y$ be as in Problem \ref{prb:SLR} and let $(\Pi^*,\xi^*), \, \xi^* = [\xi_1^*,\dots,\xi_n^*]^\top$ be a solution. Let $p \in \Re[z]$ be a symmetric polynomial.
Then the key observation is that the equality $\Pi^* y = A \xi^*$ implies the equality $p(\Pi^*y) = p (A \xi^*)$, and since $p$ is symmetric, this in turn implies
\begin{align}
\Pi^* y = A \xi^* \stackrel{\text{$p$: symmetric}}{\Longrightarrow} p(y) =p(\Pi^* y)= p(A \xi^*). 
\end{align} That is, the symmetric polynomial $p$ serves in eliminating the unknown permutation $\Pi^*$ and providing a constraint equation that depends only on the known data $A,y$,
\begin{align}
\hat{p}(x):=p(A x) - p(y)=0, \label{eq:p-tilde-any}
\end{align} and which the solution $\xi^*$ that we seek needs to satisfy.

\begin{ex}\label{ex:p-tilde}
	Consider the data 
		\begin{align}
		A = \begin{bmatrix*}[r]
		-1 & -2 \\
		2 & -3\\
		0 & 4
		\end{bmatrix*}, \, y= \begin{bmatrix*}[r] 8 \\ -5 \\ -4 \end{bmatrix*}.
		\end{align} It is simple to check that there is only one permutation 
		\begin{align}
		\Pi^* = \begin{bmatrix}
		0 & 1 & 0\\
		0 & 0 & 1\\
		1 & 0 & 0
		\end{bmatrix}
		\end{align}
		that results in a consistent linear system of equations 
		\begin{align}
		\begin{bmatrix}
		0 & 1 & 0\\
		0 & 0 & 1\\
		1 & 0 & 0
		\end{bmatrix}
		\begin{bmatrix*}[r] 8 \\ -5 \\ -4 \end{bmatrix*}
		=
		\begin{bmatrix*}[r]
		-1 & -2 \\
		2 & -3\\
		0 & 4
		\end{bmatrix*}
		\begin{bmatrix}
		x_1 \\ 
		x_2 
		\end{bmatrix}
		\label{eq:SLR-example}
		\end{align} 
		with solution $\xi^*_1=1, \, \xi^*_2=2$; every other permutation results in inconsistent equations. Now consider the symmetric polynomial 
		\begin{align}
		p_1(z_1,z_2,z_3)=z_1+z_2+z_3,
		\end{align} which we may use as in \eqref{eq:p-tilde-any} to generate the constraint
		\begin{align}
		&(-x_1-2x_2) + (2x_1-3x_2) + 4x_2 = 8-5-4, \Leftrightarrow  \\
		& x_1-x_2=-1, \label{eq:p1-example} \\
		&\Leftrightarrow \hat{p}_1(x):=p_1(Ax)-p_1(y)=0. 
	\end{align} Indeed, we see that the solution $\xi^*=[1, \, \, 2]^\top$ satisfies \eqref{eq:p1-example}.	
\end{ex}

The polynomial $\hat{p}$ in \eqref{eq:p-tilde-any} is an element of the polynomial ring $\Re[x]$ in $n$ variables $x:=[x_1,\dots,x_n]^\top$, and the set of its roots, denoted as $\V(\hat{p}):=\{\xi \in \Re^n: \, \hat{p}(\xi)=0 \}$ and called an \emph{algebraic variety},
in principle defines a hypersurface of $\Re^n$. Since the solution $\xi^*$ to \eqref{eq:SLR} is an element of the $n$-dimensional space $\Re^n$, and $\xi^* \in \V(\hat{p})$ for any such $\hat{p}$, one expects that using $n$  \emph{sufficiently independent} such polynomials, will yield a system of $n$ equations in $n$ unknowns, 
\begin{align}
\hat{p}_1(x) = \cdots = \hat{p}_n(x) =0,
\end{align} that has a finite number of solutions. Geometrically, these solutions are the intersection points of the corresponding $n$ hypersurfaces $\V(\hat{p}_1),\dots,\V(\hat{p}_n)$, which contain all solutions to Problem \eqref{prb:SLR}, as well as potentially other irrelevant points. 

\begin{ex} \label{ex:GettingUniqueSolution}	
	Continuing with Example \ref{ex:p-tilde}, suppose we further use the symmetric polynomial
	\begin{align}
	p_2(z_1,z_2,z_3) = z_1^2+z_2^2+z_3^2
	\end{align} in \eqref{eq:p-tilde-any} to obtain the additional constraint 
	\small
	\begin{align}
	&(-x_1-2x_2)^2 +  (2x_1-3x_2)^2 + (4x_2)^2 = 8^2+(-5)^2+(-4)^2 \label{eq:p2-example}, \nonumber\\
	& \Leftrightarrow \hat{p}_2(x):=p_2(A x)-p_2(y) = 0.
	\end{align} \normalsize  Solving \eqref{eq:p1-example} with respect to $x_1$ and substituting to \eqref{eq:p2-example}, gives a quadratic equation in $x_2$ with solutions $\xi_2=2$ and $\xi_2 =-25/13$.
	Solving \eqref{eq:p1-example} for $x_1$ gives
	\begin{align}
	\V(\hat{p}_1,\hat{p}_{2})=\left\{\begin{bmatrix}1 \\2 \end{bmatrix}, \, \begin{bmatrix} -38/13 \\ -25/13 \end{bmatrix} \right\}.
	\end{align} We see that $\V(\hat{p}_1,\hat{p}_{2})$ contains the solution of the linear system \eqref{eq:SLR-example} but also an additional irrelevant point.
\end{ex}

We note here that one may use $n+1$ polynomials in order to remove the irrelevant points, e.g., as was done in \cite{Choi:ISIT18}. However, such an approach is of theoretical interest only, since a system of $n+1$ (\emph{sufficiently independent}) equations in $n$ unknowns is bound to be inconsistent even in the slightest presence of noise. Instead, here we study a system of $n$ equations in $n$ unknowns and later show (see \S \ref{section:AlgorithmicImplications}) how one can filter its roots of interest.

%%%%
\subsection{Main Results} \label{subsection:MainResults}

We present our main results in \S \ref{subsubsection:ExactData} (Theorem \ref{thm:Main}) and \S \ref{subsubsection:CorruptedData} (Theorems \ref{thm:Perturbations}-\ref{thm:Noise}) for exact and corrupted data, respectively. 

%%%
\subsubsection{Exact data} \label{subsubsection:ExactData}
As Examples \ref{ex:p-tilde}-\ref{ex:GettingUniqueSolution} suggest, a natural choice for our $n$ symmetric polynomials are the first $n$ power sums $p_k(z) \in \Re[z]:=\Re[z_1,\dots,z_m],\, k \in [n]:=\{1,\dots,n\}$, defined as 
\begin{align}
p_k(z):= z_1^k+\cdots+z_m^k. \label{eq:Pk}
\end{align} The above discussion has already established that any solution $\xi^*$ of \eqref{eq:SLR} must satisfy the polynomial constraints
\begin{align}
\hat{p}_k(x) = 0, \, \, \, \, \, \,  k \in [n], \, \, \, \text{where} \label{eq:P-exact}
\end{align}  
\begin{align}
\hat{p}_k(x) := p_k( A x)-p_k(y)=\sum_{i=1}^m (a_i^\top x)^k - \sum_{j=1}^m y_j^k, 
\end{align} and $a_i^\top$ denotes the $i$th row of $A$. The next major result guarantees that there can only be a finite number of other irrelevant solutions.

\begin{thm} \label{thm:Main}
If $A$ is generic and $y$ is some permutation of some vector in $\R(A)$, then the algebraic variety $\V(\hat{p}_1,\dots,\hat{p}_n)$ contains all $\xi^*_1,\dots,\xi^*_{\ell} \in \Re^n$ such that there exist permutations $\Pi_1^*,\dots,\Pi^*_{\ell}$ with $\Pi^*_i y = A \xi_i^*, \, \forall i \in [\ell]$, while it may contain at most $n!-\ell$ other points of $\Ce^n$. If in addition $y$ is some permutation of a generic vector in $\R(A)$, then $\ell=1$.
\end{thm} 

\noindent Theorem \ref{thm:Main} guarantees that the system of polynomial equations
\begin{align}
\hat{p}_1(x) = \cdots = \hat{p}_n(x) = 0, \label{eq:system-P-hat}
\end{align} always has a finite number of solutions in $\Ce^n$ (at most $n!$), among which lie all possible solutions $\xi^*_1,\dots,\xi^*_{\ell} \in \Re^n$ of Problem \ref{prb:SLR}. The importance of the solutions being finite in $\Ce^n$ is computational: even if one is interested only in the real roots, knowing that the system has finitely many complex roots allows one to use much more efficient solvers. On the other hand, there exist many pathological cases where a system of polynomial equations has finitely many real roots but an infinity of complex roots, as the next example demonstrates; Theorem \ref{thm:Main} guarantees that such a pathological case can not occur.

\begin{ex} \label{ex:InfiniteComplex}
The polynomial equation $x_1^2+x_2^2=0$ has only one real root $[0, \, \, 0]^\top$, while over the complex numbers it defines a union of two lines in $\Ce^2$.
\end{ex} 

%%%
\subsubsection{Corrupted data} \label{subsubsection:CorruptedData}
We next consider corrupted versions $\tilde{A},\tilde{y}$ of $A,y$ respectively, the case of interest in practical applications. 
Define the \emph{corrupted} power-sum polynomials as
\begin{align}
\tilde{p}_k(x) := p_k(\tilde{A}x) - p_k(\tilde{y}), \, \, \, k \in [n], \label{eq:p-tilde}
\end{align} and consider the polynomial system $\tilde{\P}$ of $n$ equations of degrees $1,2,\dots,n$ in $n$ unknowns, given by 
\begin{align}
\tilde{\P}: \, \, \, \tilde{p}_1 = \cdots = \tilde{p}_n = 0. \label{eq:PolynomialSystem-tilde}
\end{align} A natural question is: how do the roots of $\tilde{\P}$ behave as functions of the corruptions $\tilde{A}-A, \tilde{y}-y$? This is a challenging question to answer analytically. But in fact, there is an even more fundamental question lurking in our development: does $\tilde{\P}$ have any roots at all in $\Ce^n$? As the next example shows, the answer is not necessarily affirmative.  

\begin{ex}\label{ex:P-tilde-no-solution}	
	Let $A,y$ be as in Example \ref{ex:p-tilde}, and let 
	\begin{align}
	\tilde{A}=\begin{bmatrix*}[r]
	-2 & -1 \\
	2 & -3\\
	0 & 4
	\end{bmatrix*}
	\end{align} be a corrupted version of $A$ resulting from swapping the two elements in the first row of $A$. Moreover, consider no corruption on $y$, i.e., $\tilde{y}=y$. Then $\tilde{\P}$ has no solutions because $\tilde{p}_1=-1$, i.e., the polynomial constraint corresponding to $\tilde{p}_1$ becomes infeasible. 	
\end{ex}

%\noindent In general, given $n$ polynomial equations in $n$ unknowns it is very complicated to determine whether a solution exists or not. A criterion for determining such a consistency is given by Hilbert's Nullstellensatz (Proposition \ref{prp:Nullstellensatz}), which can be checked algorithmically via the device of \emph{Gr\"obner basis}. 

A second fundamental question is: Suppose $\tilde{\P}$ has a solution, then does it have finitely many solutions? Once again, this need not be true, as the following example illustrates.

\begin{ex}\label{ex:P-tilde-infinite-solutions}
Let $\tilde{A}$ be as in Example \ref{ex:P-tilde-no-solution} and  
\begin{align}
\tilde{y}=\begin{bmatrix*}[r] 9 \\ -5 \\ -4 \end{bmatrix*}= y + \begin{bmatrix*}[r] 1 \\ 0 \\ 0 \end{bmatrix*}.
\end{align} Then $\tilde{\P}$ has infinitely many solutions because $\tilde{p}_1$ is identically the zero polynomial so that the root locus of $\tilde{\P}$ is the curve $\V(\tilde{p}_2) = \{\xi \in \Ce^n: \, 8\xi_1^2-2 \xi_1 \xi_2 +26 \xi_2^2-105=0\}$.	
\end{ex}

\noindent Theorem \ref{thm:Perturbations} below essentially states that the pathological situations of Examples \ref{ex:P-tilde-no-solution}-\ref{ex:P-tilde-infinite-solutions} can only occur for $\tilde{A}$ taking values on a subset of $\Re^{m \times n}$ of measure zero, regardless of what $\tilde{y}$ is.

\begin{thm} \label{thm:Perturbations}
If $\tilde{A}$ is generic and $\tilde{y} \in \Re^m$ is any vector, then $\V(\tilde{p}_1,\dots,\tilde{p}_n)$ is non-empty containing at most $n!$ points of $\Ce^n$. 
\end{thm} 

\noindent Due to the linearity of the linear regression model it is customary to restrict attention to the case where the corruptions affect only the observations. Then a consequence of Theorem \ref{thm:Perturbations} and \cite{Luo:SIAM94} is:

\begin{thm} \label{thm:Noise}
Suppose the corruptions affect only the observations, i.e. $\tilde{A}=A$ and $\tilde{y}= y + \varepsilon$, for $\varepsilon \in \Re^m$. If $A$ is generic then $\V(\tilde{p}_1,\dots,\tilde{p}_n) = \{ \tilde{\xi}_1, \dots, \tilde{\xi}_L\} \subset \Ce^n, \, 1 \le L \le n!$. Moreover, there exist positive constants $\tau, \gamma$ and non-negative constant $\gamma'$ such that for every $\xi^* \in \Re^n$ for which there exists a permutation $\Pi^*$ with $\Pi^*y =  A \xi^*$, there exists $i^* \in [L]$ such that 
\begin{align}
\|\xi^* - (\tilde{\xi}_{i^*})_{\Re}\|_2 \le \tau \, (1 + \|\xi^*\|_2)^{\gamma'} \, \|e\|^{\gamma}_2,  \label{eq:noisy-bound}
\end{align} where $(\tilde{\xi}_{i^*})_{\Re}$ denotes the real part of $\tilde{\xi}_{i^*}$ and $e \in \Re^n$ with $e_k=\sum_{j=1}^m \sum_{\ell =1}^ k {k \choose \ell} y_j^{k-\ell} \varepsilon_j^{\ell}$ for every $k \in [n]$.  
\end{thm}

\noindent Theorems \ref{thm:Perturbations}-\ref{thm:Noise} are important for several reasons. First, they guarantee that the system of polynomial equations \eqref{eq:PolynomialSystem-tilde} is almost always consistent, i.e., there exists at least one solution. In the absence of noise this property is immediate simply because $\xi^*$ is a root to any of the noiseless polynomials $\hat{p}_k$. However, for noisy data the consistency of \eqref{eq:PolynomialSystem-tilde} is far from obvious; for example, the authors of \cite{Abid:arXiv17} write \emph{It is generally impossible to solve these equations analytically; in fact there may not be a solution to the system}. Theorems \ref{thm:Perturbations}-\ref{thm:Noise} guarantee that such an issue is of no concern. Secondly, they guarantee finiteness of solutions in $\Ce^n$: this is important because it allows the deployment of efficient polynomial solvers specifically designed for \emph{zero-dimensional} polynomial systems; moreover, it gives us the option of seeking an estimate for the parameters $\xi^*$ of the regression problem among a finite set of points. Indeed, the bound of Theorem \ref{thm:Noise} implies that the system is \emph{H\"older continuous} \cite{Luo:SIAM94}, so that for a \emph{small} level of observation noise $\varepsilon$, some root of the noisy system \eqref{eq:PolynomialSystem-tilde} is expected to be \emph{close} to $\xi^*$ (the entries of $e$ in \eqref{eq:noisy-bound} are polynomials in $\varepsilon$ which evaluate to zero for $\varepsilon=0$). Admittedly, characterizing the SNR rate of such an estimator requires tight bounds on the constants $\tau,\gamma, \gamma'$. We leave this problem to future research with the notes that 1) as per \cite{Luo:SIAM94} \emph{the exponent $\gamma$ is typically a small (and unknown) positive number in contrast to the exponent $\gamma=1$}, and 2) the encouraging simulations of \S \ref{section:AlgorithmicImplications} suggest that $\tau,\gamma, \gamma'$ are well-behaved. 

%%%%
\subsection{Proofs} \label{subsection:Proofs}

%%%
\subsubsection{Proof of Theorem \ref{thm:Uniqueness}} \label{subsubsection:Proof-Proposition}

Let $A$ be a generic $m \times n$ matrix with $m>n$ and $\xi$ an $n \times 1$ generic vector. Since $A$ is generic and $m>n$, the rank of $A$ is equal to $n$. Let $\Pi^*$ be a permutation as in Problem \ref{prb:SLR}. We want to show that the only permutation $\Pi$ for which $\Pi (\Pi^*)^\top A \xi$ is in the range-space of $A$ is $\Pi=\Pi^*$. This is equivalent to proving that for any permutation $\Pi$ different than the identity we have
\begin{align}
\rank [A \, \, \, \Pi A \xi] = n+1.
\end{align}

\noindent Since $A$ is generic, it can be written as the product $A = L U$ of an $m \times m$ generic lower triangular matrix $L$ and an $m \times n$ matrix $U$, whose top $n \times n$ block is a generic upper triangular matrix and its $(m-n) \times n$ bottom part is the zero matrix. Then 
\begin{align}
L^{-1} [A \, \, \, \Pi A \xi] = [U \, \, \, L^{-1} \Pi L U \xi].
\end{align} Because of the structure of $U$ it is enough to show that one of the last $m-n$ entries of the vector $L^{-1} \Pi L U \xi$ is non-zero. But because $\xi$ is generic, it is enough to show that one of the last $m-n$ rows of the matrix 
$L^{-1} \Pi L U$ is non-zero. Towards that end, we will show that the $(m,1)$ entry of this matrix is non-zero. This entry is zero if and only if 
\begin{align}
L^{-1}_{m,:} \Pi L_{:,1} = 0,
\end{align} where $L^{-1}_{m,:}$ denotes the last row of $L^{-1}$ and $L_{:,1}$ denotes the first column of $L$. This last equation says that $L^{-1}_{m,:}$ must be orthogonal to $\Pi L_{:,1}$. But by definition, $L^{-1}_{m,:}$ is orthogonal to all the columns $L_{:,1},\dots,L_{:,m-1}$ of $L$ except the last one. Put together, we have that $L^{-1}_{m,:}$ must be orthogonal to $L_{:,1},\dots,L_{:,m-1},\Pi L_{:,1}$, which is possible only if 
\begin{align}
\rank [L_{:,1} \, \,  \cdots \, \, L_{:,m-1} \, \, \Pi L_{:,1}] = m-1.
\end{align} Hence, it is enough to show that 
\begin{align}
\det [L_{:,1} \, \,  \cdots \, \, L_{:,m-1} \, \, \Pi L_{:,1}] \neq 0. \label{eq:det-L}
\end{align} Towards that end, we view the entries $\ell_{ij}$ of $L$ as variables of a polynomial ring $\Ce[\ell_{11},\dots,\ell_{mm}]$ in $m(m+1)/2$ variables. Then the determinant in \eqref{eq:det-L} is a polynomial of $\Ce[\ell_{11},\dots,\ell_{mm}]$. 
It is enough to show that this polynomial is non-zero. For in that case, it defines a hypersurface of $\Ce^{m(m+1)/2}$, which will not contain $L$, since $L$ is generic. To show that the determinant is indeed a non-zero polynomial we consider an ordering of the variables
\begin{align}
\ell_{11} > \ell_{21} > \ell_{22} > \ell_{31} > \ell_{32}>\cdots>\ell_{m-1,m}>\ell_{mm},
\end{align} and then consider the induced lexicographic order on all the monomials of $\Ce[\ell_{11},\dots,\ell_{mm}]$. 

First, suppose that $\ell_{11}$ occurs in the $k$th entry of $\Pi L_{:,1}$ where $k>1$. We will show that the largest monomial appearing in the definition of $\det [L_{:,1} \, \,  \cdots \, \, L_{:,m-1} \, \, \Pi L_{:,1}]$ only occurs in one way; this guarantees that the determinant is non-zero. Indeed, the largest monomial occurs in the following unique way: choose the largest element from each row, starting from the rows that contain the largest elements. That is, pick element $\ell_{11}$ from row $1$, element $\ell_{11}$ from row $k$, element $\ell_{ii}$ from row $i$ for $i=2,\dots,k-1$, and elements $\ell_{i+1,i}$ from row $i+1$ for $i=k+1,\dots,m-1$. 

Next, suppose that $k \ge 1$ is the largest index such that the $i$th entry of $\Pi L_{:,1}$ is equal to $\ell_{i1}$ for $i=1,\dots,k$. In that case, we apply an elementary column operation by subtracting the first column of $[L_{:,1} \, \,  \cdots \, \, L_{:,m-1} \, \, \Pi L_{:,1}]$ from its last, to obtain a new matrix $M$ of the same rank. Then the first $k$ entries of the last column of $M$ are zero, while its $(i,m)$ entry for $i>k$ is of the form $\ell_{s_i,1}-\ell_{i,1}$, with $s_i>k, \, \forall i=k+1,\dots,m$, and with $s_{k+1} \neq k+1$. Let $t>1$ be such that $s_{k+t} = k+1$. Then the largest monomial in $\det(M)$ occurs in a unique way as the first term of the expanded product of elements $\ell_{ii}$ from row $i$ for $i=1,\dots,k+t-1$, element $(\ell_{k+1,1}-\ell_{k+t,1})$ from row $k+t$, and elements $\ell_{i+1,i}$ from row $i+1$ for $i=k+t,\dots,m-1$. That is, $\det(M)$ is a non-zero polynomial.

%%%
\subsubsection{Power-sums and regular sequences} \label{subsubsection:Proof-Preliminaries}

The goal of this section is to prove Lemma \ref{lem:Pbar}, which is the main tool needed for the proof of Theorems \ref{thm:Main} and \ref{thm:Perturbations}; the reader is encouraged to consult Appendix \ref{appendix:regular-sequences} before proceeding. To begin with, recall the power-sums polynomials 
\begin{align}
p_k = z_1^k+\cdots+z_m^k, \, \, \, k: \, \text{positive integer}, \label{eq:PowerSums}
\end{align} that were used in \S\ref{subsection:MainResults} as the base symmetric polynomials towards eliminating the unknown permutation $\Pi^*$. These are polynomials in $m$ variables $z_1,\dots,z_m$, i.e., $p_k \in \Ce[z]:=\Ce[z_1,\dots,z_m]$. The next two lemmas establish that $p_1,\dots,p_m$ form a regular sequence\footnote{This also follows from a more general theorem in \cite{Conca:Padova2009}, which states that any $m$ consecutive such polynomials $p_{\ell},p_{\ell+1},\dots,p_{\ell+m-1}$, where $\ell$ is any positive integer, form a regular sequence of $\Ce[z]$. To make the paper more accessible and self-contained, we have taken the liberty of giving the rather simple case $\ell=1$ its own proof.}. 

%%%
\begin{lem} \label{lem:DimensionElementarySymmetric}
Let $\sigma_1,\dots,\sigma_m$ be the elementary symmetric polynomials in $m$ variables $z$, defined as 
\begin{align}
\sigma_k&:= \sum_{1 \le i_1<i_2<\cdots<i_k\le m} z_{i_1} z_{i_2} \cdots z_{i_k}.
\end{align} Then $\sigma_1,\dots,\sigma_m$ form a regular sequence of $\Ce[z]$. 
\end{lem}
\begin{proof}
We will show that $\V_{\Ce^m}(\sigma_1,\dots,\sigma_m) = \{0\}$, in which case the statement will follow from Proposition \ref{prp:RegularSequences} together with the fact that each $\sigma_k$ is homogeneous (of degree $k$). 
We proceed by induction on $m$. If $m=1$, $\Ce[z] = \Ce[z_1]$ and the only elementary symmetric polynomial is $z_1$. Clearly, $z_1$ vanishes only on $0 \in \Ce$. Suppose now that $m>1$. 
Let $\zeta=[\zeta_1,\dots,\zeta_m]^\top \in \V(\sigma_1,\dots,\sigma_m)$. Then $\sigma_m(\zeta) = \zeta_1 \cdots \zeta_m=0$, and without loss of generality we can assume that $\zeta_m=0$. 
Then $[\zeta_1,\dots,\zeta_{m-1}]^\top$ is in the variety generated by the $m-1$ elementary symmetric polynomials in $m-1$ variables and the induction on $m$ gives $\zeta_1=\cdots=\zeta_{m-1}=0$.
\end{proof}

%%%
\begin{lem} \label{lem:PowerSumsRegularSequence}
The first $m$ power-sum polynomials $p_1,\dots,p_m$ defined in \eqref{eq:PowerSums} form a regular sequence of $\Ce[z]$. \end{lem}
\begin{proof}
Newton's identities 
\begin{align}
k \sigma_k = \sum_{i=1}^k (-1)^{i-1} \sigma_{k-i} p_i, \, \, \, \forall  k \in [m],
\end{align} show that $\sigma_1,\dots,\sigma_m$ can be obtained inductively in terms of $p_1,\dots,p_m$. On the other hand, the fundamental theorem of symmetric polynomials states that every symmetric polynomial can be written as a polynomial in $\sigma_1,\dots,\sigma_m$. This implies that we have an equality of ideals
$(p_1,\dots,p_m) = (\sigma_1,\dots,\sigma_m)$, hence an equality of algebraic varieties $\V_{\Ce^m}(p_1,\dots,p_m) = \V_{\Ce^m}(\sigma_1,\dots,\sigma_m)$. Thus Lemma \ref{lem:DimensionElementarySymmetric} gives that $\V_{\Ce^m}(p_1,\dots,p_m)$ consists of a single point and Proposition \ref{prp:RegularSequences} together with the fact that each $p_k$ is homogeneous (of degree $k$) establishes that $p_1,\dots,p_m$ is a regular sequence.
\end{proof}

Now, recall that in \S \ref{subsection:MainResults} each \emph{base} polynomial $p_k$ was used to furnish a polynomial equation $\hat{p}_k(x)=p_k(A x) - p_k(y)=0$ that the unique solution $\xi^*$ to Problem \ref{prb:SLR} should satisfy. Notice here how we are passing from polynomials $p_k \in \Ce[z]$ in $m$ variables to polynomials $\hat{p}_k \in \Ce[x]$ in $n$ variables. Notice further that even though the $p_k$ are symmetric and homogeneous, the $\hat{p}_k$ are in principle neither symmetric nor homogeneous. In fact,  
\begin{align}
\hat{p}_k &= \bar{p}_k - p_k(y), \, \, \, \text{where}\\
\bar{p}_k &:= p_k(A x)= \sum_{i=1}^m (a_i^\top x)^k , \, \, \, \, \, \, k=1,\dots,n.\label{eq:Pbar}
\end{align} Here $\bar{p}_k$ is a homogeneous polynomial of degree $k$ and $p_k(y) \in \Re$ is a constant. As it turns out, the homogeneous parts $\bar{p}_k$ of the $\hat{p}_k$ are the bridge for passing properties of the base polynomials $p_k$ to the polynomials of interest $\hat{p}_k$. The next two lemmas are the two required steps towards building that bridge.

\begin{lem}\label{lem:VarietyIsomorphism}
Let $\ell_1^\top,\dots,\ell_{m-n}^\top \in \Re^{1 \times m}$ be a basis for the left nullspace of $A$. Then the algebraic varieties $\V_{\Ce^n}(\bar{p}_1,\dots,\bar{p}_n)$ and $\V_{\Ce^m}(\ell_1^\top z, \dots, \ell_{m-n}^\top z, p_1,\dots,p_n)$ are isomorphic. In particular, the two varieties have the same dimension.
\end{lem}
\begin{proof}
Let $\X:=\V_{\Ce^m}(\ell_1^\top z, \dots, \ell_{m-n}^\top z, p_1,\dots,p_n)$ and $\Y:= \V_{\Ce^n}(\bar{p}_1,\dots,\bar{p}_n)$. To show that $\X$ and $\Y$ are isomorphic as algebraic varieties, it is enough to find a bijective map $f: \Y \rightarrow \X$, such that both $f$ and $f^{-1}$ are given by polynomials. First, we define $f$ by specifying its image on an arbitrary point $\xi$ of $\Y$:
\begin{align}
\xi \in \Y \subset \Ce^n \stackrel{f}{\longmapsto} A \xi \in \Ce^m. 
\end{align} We need to show that $f(\xi) \in \X$. To do that, we need to check that $f(\xi)$ satisfies the defining equations of $\X$. Towards that end, note that $\ell_j^\top f(\xi) = \ell_j^\top A \xi = 0, \, \forall j \in [m-n]$, by definition of the $\ell_j$. Moreover, $p_k(f(\xi)) = p_k(A \xi) = \bar{p}_k(\xi) = 0, \, \forall k \in [n]$, where the last equality is true because $\xi \in \Y$. Hence $f$ is indeed a map of the form $f: \Y \rightarrow \X$. Moreover, $f$ is given by linear polynomials, i.e., the coordinates of $f(\xi)$ are linear polynomials of the coordinates of $\xi$. 

Next, we define a map $g: \X \rightarrow \Y$ as follows. Let $\zeta \in \X$. Then $\ell_j^\top \zeta = 0, \, \forall j \in [m-n]$. Since the $\ell_j$ form a basis for $\N(A^\top)$, this means that $\zeta \in \N(A^\top)^\perp$. But from basic linear algebra $\N(A^\top)^\perp = \R(A)$. Hence there exists some $\xi \in \Ce^n$ such that $\zeta = A \xi$. Moreover, this $\xi$ is unique because $A$ has full rank by the assumption of Problem \ref{prb:SLR}. This allows us to define the map $g$ as $g(\zeta) =\xi:= (A^\top A)^{-1}A^\top \zeta$. Then 
\begin{align}
\bar{p}_k(g(\zeta)) &= p_k(A g(\zeta)) = p_k(A (A^\top A)^{-1}A^\top \zeta) \\
&= p_k(A (A^\top A)^{-1}A^\top A \xi) \\
&= p_k(A \xi) \\
&= p_k(\zeta) = 0,
\end{align} where the last equality is true because $\zeta \in \X$. Hence $g(\zeta) \in \Y$. Moreover, $g$ is given by polynomials, since each coordinate of $g(\zeta)$ is a linear polynomial in the coordinates of $\xi$. Finally, it is simple to check that $f$ and $g$ are inverses of each other.
\end{proof}
 
%%%
\begin{lem} \label{lem:Pbar}
If $A$ is generic, then the polynomials $\bar{p}_1,\dots,\bar{p}_n$ defined in \eqref{eq:Pbar} form a regular sequence of $\Ce[x]$.
\end{lem}
\begin{proof}
The space of all linear subspaces of dimension $n$ of $\Re^m$ is an algebraic variety of $\Re^M, \, M=\binom{m}{n}$, called \emph{Grassmannian} and denoted $\mathbb{G}(n,m)$. The space $\mathbb{G}(m-n,m)$ of all linear subspaces of dimension $m-n$ of $\Re^m$ is also an algebraic variety of $\Re^M$. These two varieties are isomorphic under a mapping that takes a linear subspace $\cS$ of dimension $n$ to its orthogonal complement $\cS^\perp$. Hence, $\cS$ is generic if and only if $\cS^\perp$ is generic. Choosing $A$ generic is the same as choosing $n$ generic vectors in $\Re^m$, which is the same as choosing a generic subspace $\cS=\R(A)$ of $\Re^m$. This is the same as choosing $\cS^\perp=\N(A^\top)$ generic, which is the same as choosing $m-n$ generic vectors of $\Re^m$. In other words, if $A$ is generic, then a generic set of vectors $\ell_1, \dots, \ell_{m-n}$ inside $\N(A^\top)$, is a set of $m-n$ generic vectors of a generic $(m-n)$-dimensional linear subspace of $\Re^m$. Hence, with respect to any property that does not depend on $A$, the vectors $\ell_1, \dots, \ell_{m-n}$ behave as generic vectors of $\Re^m$.

So let $L= [\ell_1 \, \cdots \, \ell_{m-n}]$ be a matrix containing in its columns a set of generic vectors of $\N(A^\top)$.
By Lemma \ref{lem:PowerSumsRegularSequence} $p_1,\dots,p_m$ is a regular sequence of $\Ce[z]$. By the definition of regular sequence, the subsequence $p_1,\dots,p_n$ is also regular. By inductive application of Proposition \ref{prp:GenericHyperplanes},
\begin{align}
p_1,\dots,p_n,\ell_1 ^\top z,\dots,\ell_{m-n} ^\top z
\end{align} is also a regular sequence of $\Ce[z]$. By Proposition \ref{prp:RegularSequences}, we have that 
\begin{align}
\dim \V_{\Ce^m}(\ell_1^\top z, \dots, \ell_{m-n}^\top z, p_1,\dots,p_n) = 0.
\end{align} By Lemma \ref{lem:VarietyIsomorphism} we have that 
$\dim \V_{\Ce^n}(\bar{p}_1,\dots,\bar{p}_n) = 0$. Then by Proposition \ref{prp:RegularSequences} $\bar{p}_1,\dots,\bar{p}_n$ is a regular sequence of $\Ce[x]$.
\end{proof}

\subsubsection{Proof of Theorem \ref{thm:Main}}
%%%

Having developed the machinery that led to Lemma \ref{lem:Pbar}, the task of proving Theorem \ref{thm:Main} is a matter of putting this machinery to work; at this stage the reader is encouraged to consult appendices \ref{appendix:geometric-dimension}-\ref{appendix:initial-ideals} before proceeding. 

Let $\I=(\hat{p}_1,\dots,\hat{p}_n)$ be the ideal generated by our polynomials $\hat{p}_1,\dots,\hat{p}_n$. Consider the weight order of $\Ce[x]$ defined by the vector $w=[1,\dots,1]^\top \in \mathbb{Z}^n$ (see Appendix \ref{appendix:initial-ideals}). Since $\bar{p}_1,\dots,\bar{p}_n$ is a regular sequence of homogeneous polynomials (Lemma \ref{lem:Pbar}) and $\hat{p}_k=\bar{p}_k+\text{constant}$, we can obtain $\In_{w}(\I)$ just from the leading homogeneous terms of the generators of $\I$, i.e., $\In_{w}(\I) = (\bar{p}_1,\dots,\bar{p}_n)$ (Proposition \ref{prp:InitialRegular}). Since $\bar{p}_1,\dots,\bar{p}_n$ is a regular sequence of length $n$ in the polynomial ring $\Ce[x]$ of $n$ variables, we have $\dim \V_{\Ce^n}(\In_{w}(\I)) = 0$ (Proposition \ref{prp:RegularSequences}). Since $\V_{\Ce^n}(\I)$ and $\V_{\Ce^n}(\In_{w}(\I))$ have the same dimension  (Proposition \ref{prp:InitialIdeal-Dimension}), we have $\dim \V_{\Ce^n}(\I) = 0$. Since zero-dimensional varieties have a finite number of points (Proposition \ref{prp:ZeroDimensionFinitePoints}), we have that $\V_{\Ce^n}(\hat{p}_1,\dots,\hat{p}_n)$ consists of finitely many points.

We now bound from above the number of points of $\Y:=\V_{\Ce^n}(\hat{p}_1,\dots,\hat{p}_n)$. 
Consider the affine cone $\Y^{(h)}=\V_{\Ce^{n+1}}(\hat{p}_1^{(h)},\dots,\hat{p}_n^{(h)})$ of $\Y$ (see Appendix  \ref{subsubsection:Homogenization}). Since $\dim \Y=0$, we have $\dim \Y^{(h)}=1$ (Proposition \ref{prp:AffineConeDimension}). 
Let $\Y^{(h)} = \W_1 \cup \cdots \cup \W_{\ell}$ be the irreducible decomposition of $\Y^{(h)}$ (Proposition \ref{prp:IrreducibleDecomposition}). 
Since $\Y^{(h)}$ is generated by homogeneous polynomials, it is the union of lines through the origin. 
Hence, each irreducible component $\W_i$ is the union of lines through the origin. We argue that each $\W_i$ is a single line through the origin. For if some $\W_j$ is not a single line, let $\xi \in \W_j$ be a point different than the origin $0$. Letting $\L$ be the line through the origin and $\xi$, we have a chain $\W_j \supsetneq \L \supsetneq \{0\}$ of irreducible subsets of $\Y^{(h)}$ of length $2$. 
But this contradicts the fact that $\dim \Y^{(h)}=1$ (Definition \ref{dfn:DimensionGeometric}). 
This shows that $\Y^{(h)}$ is the union of $\ell$ lines through the origin. Then 
$\ell \le \deg(\hat{p}_1^{(h)}) \deg(\hat{p}_2^{(h)}) \cdots \deg(\hat{p}_n^{(h)}) = n!$ (Proposition \ref{prp:Bezout}). 
Finally, the points of $\Y$ are in $1-1$ correspondence with the intersection points of $\Y^{(h)}$ with the hyperplane $t=1$ of $\Ce^{n+1}$ 
(see Appendix \ref{subsubsection:Homogenization}). But the number of these intersection points can not exceed the number of lines of $\Y^{(h)}$.

%%%%%%%%%%
\subsubsection{Proof of Theorem \ref{thm:Perturbations}}

The proof is almost the same as the proof of Theorem \ref{thm:Main} with an additional twist: we need to show that the variety $\tilde{\Y}:= \V_{\Ce^n}(\tilde{p}_1,\dots,\tilde{p}_n)$, where the $\tilde{p}_k$ are defined in \eqref{eq:p-tilde}, is non-empty. So suppose that $\tilde{\Y}=\emptyset$. Then $1 \in \tilde{I}:=(\tilde{p}_1,\dots,\tilde{p}_n)$ (Proposition \ref{prp:Nullstellensatz}). Hence $1=\In_{w}(1) \in \In_{w}(\tilde{I})$. Since $\tilde{A}$ is generic, the polynomials $\tilde{\bar{p}}_k := p_k(\tilde{A}x), \, k \in [n]$, form a regular sequence (Lemma \ref{lem:Pbar}). Since $\tilde{p}_k:=  \tilde{\bar{p}}_k-p_k(\tilde{y})$, we must have that $\In_{w}(\tilde{I}) = (\tilde{\bar{p}}_1,\dots,\tilde{\bar{p}}_n)$ (Proposition \ref{prp:InitialRegular}). But now $1 \in (\tilde{\bar{p}}_1,\dots,\tilde{\bar{p}}_n)$ is impossible, because all $\tilde{\bar{p}}_k$ are homogeneous of positive degree, i.e., the equation $1 = q_1 \tilde{\bar{p}}_1+\cdots+q_n\tilde{\bar{p}}_n$ can not be true for any $q_k \in \Ce[x]$. This contradiction means that $\tilde{\Y}$ is non-empty. The rest of the proof is the same as that of Theorem \ref{thm:Main}.

%%%%%%%%%%
\subsubsection{Proof of Theorem \ref{thm:Noise}}

In view of Theorem \ref{thm:Perturbations} we only need to prove \eqref{eq:noisy-bound}. For $\xi \in \Ce^n$ write $\xi = (\xi)_{\Re} + \sqrt{-1} \, (\xi)_{\mathbb{I}}$, where $(\xi)_{\Re},(\xi)_{\mathbb{I}}$ are the real and imaginary parts of $\xi$ respectively. Let $\Re[u,v]=\Re[u_1,\dots,u_n,v_1,\dots,v_n]$ be a  polynomial ring in $2n$ variables. Then for each $\bar{p}_k$ (defined in \eqref{eq:Pbar}) there exist unique polynomials $\bar{g}_k,\bar{h} _k \in \Re[u,v]$, such that 
\begin{align}
&\bar{p}_k(\xi) = \bar{g}_k \big( (\xi)_{\Re}, (\xi)_{\mathbb{I}}\big) + \sqrt{-1} \,  \bar{h}_k \big( (\xi)_{\Re}, (\xi)_{\mathbb{I}}\big), \\ 
& \forall \xi=(\xi)_{\Re} + \sqrt{-1} \, (\xi)_{\mathbb{I}} \in \Ce^n. \nonumber
\end{align} Define $\hat{g}_k = \bar{g}_k - \sum_{j=1}^m y_j^k$ and $\tilde{g}_k = \bar{g}_k - \sum_{j=1}^m (y_j+\varepsilon_j)^k$. Then 
\begin{align}
&\hat{p}_k(\xi) = \hat{g}_k \big( (\xi)_{\Re}, (\xi)_{\mathbb{I}}\big) + \sqrt{-1} \,  \bar{h}_k \big( (\xi)_{\Re}, (\xi)_{\mathbb{I}}\big), \\  
&\tilde{p}_k(\xi) = \tilde{g}_k \big( (\xi)_{\Re}, (\xi)_{\mathbb{I}}\big) + \sqrt{-1} \,  \bar{h}_k \big( (\xi)_{\Re}, (\xi)_{\mathbb{I}}\big), \\
& \forall \xi=(\xi)_{\Re} + \sqrt{-1} \, (\xi)_{\mathbb{I}} \in \Ce^n. \nonumber
\end{align} We thus have the following set-theoretic equalities, where we identify $\xi=(\xi)_{\Re} + \sqrt{-1} \, (\xi)_{\mathbb{I}} \in \Ce^n$ with $\big((\xi)_{\Re},(\xi)_{\mathbb{I}} \big) \in \Re^{2n}$:
\begin{align}
\V_{\Ce^n}(\hat{p}_1,\dots,\hat{p}_n) &= \V_{\Re^{2n}}(\hat{g}_1,\dots,\hat{g}_n,\bar{h}_1,\dots,\bar{h}_n), \\
\V_{\Ce^n}(\tilde{p}_1,\dots,\tilde{p}_n) &= \V_{\Re^{2n}}(\tilde{g}_1,\dots,\tilde{g}_n,\bar{h}_1,\dots,\bar{h}_n).
\end{align} Let $\xi^* \in \Re^n$ be such that there exists a permutation $\Pi^*$ for which $\Pi^* y = A \xi^*$. Then $\hat{p}_k(\xi^*) = 0$ for every $k \in [n]$. Equivalently, for $(\xi^*,0) \in \Re^{2n}$ we have $\hat{g}_k(\xi^*,0) = \bar{h}_k(\xi^*,0)=0$. Moreover, 
\begin{align}
\tilde{g}_k(\xi^*,0) &= \bar{g}_k(\xi^*,0) - \sum_{j=1}^m (y_j+\varepsilon_j)^k \\
&= \hat{g}_k(\xi^*,0) - \sum_{j=1}^m \sum_{\ell=1}^ k {k \choose \ell} y_j^{k-\ell} \varepsilon_j^\ell \\
&=- \sum_{j=1}^m \sum_{\ell=1}^ k {k \choose \ell} y_j^{k-\ell} \varepsilon_j^\ell. 
\end{align}
Now, Theorem 2.2 in \cite{Luo:SIAM94} guarantees the existence of constants $\tau, \gamma, \gamma'$ as claimed for which there exists $\xi\in \V_{\Ce^n}(\tilde{p}_1,\dots,\tilde{p}_n)$ such that
\begin{align}
&\|\alpha - (\xi)_{\Re} \|_2 \le\|(\alpha,\beta) - \big((\xi)_{\Re},(\xi)_{\mathbb{I}} \big) \|_2 \le \\
&\tau \, (1+ \|(\alpha,\beta)\|_2)^{\gamma'} \|e (\alpha,\beta)\|_2^\gamma, \, \, \, \forall \alpha, \beta \in \Re^n,
\end{align} where 
\small
\begin{align}
e(\alpha,\beta) = \big(\tilde{g}_1(\alpha,\beta),\dots,\tilde{g}_n(\alpha,\beta),\bar{h}_1(\alpha,\beta),\dots,\bar{h}_n(\alpha,\beta) \big). 
\end{align} \normalsize Setting $(\alpha,\beta) = (\xi^*,0)$ concludes the proof.

%%%%%%%%%%%%%%%%%
\section{Algorithm: Algebraically Initialized \\ Expectation Maximization}\label{section:AlgorithmicImplications}

%%%%%%%
\subsection{Algorithm Design} \label{subsection:AI-EM}

Assuming for simplicity (and without much loss of generality as per Theorem \ref{thm:Uniqueness}) that there is a unique solution $\xi^*$ to Problem \ref{prb:SLR}, Theorem \ref{thm:Main} guarantees that $\xi^*$ is one of the finitely many complex roots of the noiseless polynomial system $\P$ of $n$ equations in $n$ unknowns given in \eqref{eq:P-exact}. Even if the data $\tilde{y},\tilde{A}$ are corrupted, Theorem \ref{thm:Perturbations} further guarantees that the noisy system $\tilde{\P}$ given in \eqref{eq:PolynomialSystem-tilde} remains consistent with $L \le n!$ complex roots. If in addition the corruption level is mild, as per Theorem \ref{thm:Noise} we expect one of the roots of $\tilde{\P}$ to be a good approximation to $\xi^*$. Our goal is to isolate that root and refine it.

The polynomial system $\tilde{\P}$ (\eqref{eq:PolynomialSystem-tilde}) is solvable in closed form for $n=1,2$, while for $n \ge 3$ it  can be solved using any state-of-the-art solver \cite{Lazard:JSC2009}. Nevertheless, the complexity of solving such a system is known to be exponential in $n$, even for well-behaved cases \cite{Faugere:JSC2016}. Practically speaking, this currently limits us to the regime $n \le 6$, since for $n=6$ a standard homotopy-based solver such as Bertini \cite{Bertini} takes about $37$ minutes on an Intel(R) i7-8650U, 1.9GHz, 16GB machine. On the other hand, for $n=3,4$ using an automatic solver generator \cite{Kneip:Polyjam2015} along the lines of \cite{Kukelova:ECCV08} for the specific structure of our system, we are able to obtain very efficient linear algebra based solvers that run in milliseconds. 

Having obtained the roots $\tilde{\xi}_1,\dots,\tilde{\xi}_L \in {\Ce}^n, \, L \le n!$, of $\tilde{\P}$, we retain only their real parts $(\tilde{\xi}_1)_{\Re},\dots,(\tilde{\xi}_L)_{\Re}$, and identify the root that can serve as a first approximation to $\xi^*$. We do this by selecting the root that yields the smallest $\ell_2$ error among all possible permutations $\Pi$:
\begin{align}
\xi_{\text{AI}} := \argmin_{i \in [L]} \, \left\{\min_{\Pi} \left\| \Pi \tilde{y} - \tilde{A} (\tilde{\xi}_i)_{\Re} \right\|_2 \right\}. \label{eq:xi-AI}
\end{align} We note that each inner minimization $\min_{\Pi} \| \Pi \tilde{y} - \tilde{A} (\tilde{\xi}_i)_{\Re} \|_2$  in \eqref{eq:xi-AI} can be solved by sorting (see \cite{Abid:arXiv18}) and so the computation of $\xi_{\text{AI}}$ is of complexity $\mathcal{O}(L m \log(m))$. Finally, we use the \emph{algebraic initialization} $\xi_{\text{AI}}$ as an initialization to the Expectation Maximization algorithm of \cite{Abid:arXiv18}, which as noted in \S \ref{subsection:PriorArt} consists of solving \eqref{eq:MLE} via alternating minimization. The complete algorithmic listing, which we refer to as \emph{Algebraically-Initialized Expectation-Maximization (AI-EM)}\footnote{A preliminary version of this algorithm has been published in the short conference paper \cite{Peng:ICASSP19}.} is given in Algorithm \ref{alg:AI-EM}.

%%%%%%%%%
\begin{algorithm}
	\caption{Algebraically-Initialized Expectation-Maximization}\label{alg:AI-EM}	
	\begin{algorithmic}[1]
	\Procedure{AI-EM}{$\tilde{y} \in\Re^{m}, \, \tilde{A} \in \Re^{m\times n}, \, T \in \mathbb{N}$, $\epsilon \in \Re_{+}$}		
	        \State $p_k(z) := \sum_{j=1}^m z_j^k,\,  \tilde{p}_k := p_k(\tilde{A} x)-p_k(\tilde{y}), \,  k\in [n]$;	
		\State Compute roots $\{\tilde{\xi}_i\}_{i=1}^L \subset \Ce^n$ of $\{\tilde{p}_k=0, \, k \in [n] \}$;
				\State Extract the real parts $\{(\tilde{\xi}_i)_{\Re}\}_{i=1}^L$ of $\{\tilde{\xi}_i\}_{i=1}^L \subset \Ce^n$;
	\State $\{\xi_0,\Pi_0\} \gets \argmin_{\xi \in \{(\tilde{\xi}_i)_{\Re}\}_{i=1}^L, \Pi} \,  \| \Pi \tilde{y} - \tilde{A} \xi \|_2 $;
		\State $t \gets 0$, $\Delta\J \gets \infty$, $\J \gets \| \Pi_0 \tilde{y} - \tilde{A} \xi_0 \|_2$; 
		\While{$t< T$ and $ \Delta \J > \varepsilon \J$} 
			\State $t \gets t + 1$; 
			\State $\xi_t \gets \argmin_{\xi \in \Re^n }  \| \Pi_{t-1} \tilde{y} - \tilde{A} \xi \|_2 $;
			\State $\Pi_t \gets \argmin_{ \Pi} \,  \| \Pi \tilde{y} - \tilde{A} \xi_t \|_2 $;
			\State $\Delta \J \gets \J - \| \Pi_t \tilde{y} - \tilde{A} \xi_t \|_2$;
			\State $\J \gets \| \Pi_t \tilde{y} - \tilde{A} \xi_t \|_2$;
		\EndWhile	 	
		\State \textbf{Return} $\xi_t, \, \Pi_t$.	
	\EndProcedure
	\end{algorithmic}
\end{algorithm}

%%%%%
\subsection{Numerical Evaluation} \label{subsection:NumericalEvaluation}

\myparagraph{Algorithms} We perform a numerical evaluation of our proposed AI-EM Algorithm \ref{alg:AI-EM}, which relies on solving the polynomial system by homotopy continuation in Bertini \cite{Bertini} for $n \ge 5$, and on our C++ elimination-template-based custom solvers for $n=3,4$ (see \cite{Kneip:Polyjam2015,Kukelova:ECCV08} for more details).  We also compare with the robust $\ell_1$ regression method of \cite{Slawski:EJS19}, as well as two variations of EM algorithms that were proposed in \cite{Abid:arXiv18} towards computing the Maximum Likelihood Estimator of \eqref{eq:MLE}\footnote{We do not compare with the algorithm of \cite{Haghighatshoar:TSP18}, since this algorithm assumes a strong hypothesis rendering it not applicable to the case where the observation vector $y$ has not been subsampled; this is the case where unlabeled sensing becomes equivalent to linear regression without correspondences. For a comparison of \cite{Haghighatshoar:TSP18} with other unlabeled sensing methods see \cite{Tsakiris:ICML19}.}. The first, referred to as \cite{Abid:arXiv18}$^1$, computes the MLE via alternating minimization exactly as in Algorithm \ref{alg:AI-EM}, except that it uses as initialization the vector that best fits the data $\tilde{y},\tilde{A}$ in the least-squares sense. The second variation, referred to as \cite{Abid:arXiv18}$^2$, uses the same initialization as \cite{Abid:arXiv18}$^1$, but it replaces the brute search over all possible permutations by a dynamic empirical average of permutation matrices drawn from a suitable Markov chain. For all algorithms we use a maximal number of iterations $T=100$ and for AI-EM and \cite{Abid:arXiv18}$^1$ a convergence accuracy parameter of $\epsilon=0.01$. 

\myparagraph{Data} We use the following generative model with additive noise. We randomly sample $A\in \Re^{m \times n}$ and  $\xi^* \in \Re^n$ from the corresponding standard normal distributions and perform a random permutation on the entries of $A \xi^*$ to obtain a vector $y \in \Re^m$. We further corrupt $y$ by adding to it a vector $\varepsilon \in \Re^m$ sampled from the zero-mean normal distribution with covariance matrix $\sigma^2 I_{m}$. The input data then consist of $A$ and $\tilde{y}:= y+\varepsilon$. 

\myparagraph{Metrics} We assess all methods by measuring the relative estimation error of $\xi^*$, e.g., if $\xi_{\text{AI-EM}}$ is the output of AI-EM, we report  
\begin{align}
100 \frac{\|\xi^* - \xi_{\text{AI-EM}}\|_2}{\|\xi^*\|_2} \%. 
\end{align} For AI-EM we further report the estimation error that corresponds to the best root $\xi^*_{\text{AI}}$ of the polynomial system $\tilde{\P}$, defined as 
\begin{align}
\xi^*_{\text{AI}} := \argmin_{\hat{\xi}_i, \, i \in [L]} \|\xi^* - (\tilde{\xi}_i)_{\Re}\|_2,
\end{align} as well as that of our estimated best root $\hat{\xi}_{\text{AI}}$ computed as in \eqref{eq:xi-AI}. 

\myparagraph{Results} Figures \ref{fig:m500n3}-\ref{fig:m500n5} depict the estimation error of the compared methods for fully shuffled data, for $n=3,4,5$, $\text{SNR}=0:10:100\text{ dB}$ and $m=500$ fixed, averaged over $100$ independent trials. Evidently, both \cite{Abid:arXiv18}$^1$ and \cite{Abid:arXiv18}$^2$ fail\footnote{We have observed that for fully shuffled data, large $m$ and small $n$, \cite{Abid:arXiv18}$^2$ tends to drive the regression vector to zero.}. This is not surprising, since, when the data are fully shuffled, the least-squares initialization used by these methods is rather far from the ground truth $\xi^*$. The robust $\ell_1$ regression of \cite{Slawski:EJS19} fails as well, and again this is not a surprise, since with fully shuffled data the corruptions are no longer sparse. On the other hand, $\xi^*_{\text{AI}}$ remains relatively close to $\xi^*$ as $\sigma$ increases and our actual initialization $\xi_{\text{AI}}$ coincides with $\xi^*_{\text{AI}}$ for $\sigma \le 0.04$ and is slightly worse than $\xi^*_{\text{AI}}$ otherwise. Regardless, the alternating minimization further refines $\xi_{\text{AI}}$ leading to even lower errors than $\xi^*_{\text{AI}}$.  

\begin{figure*}
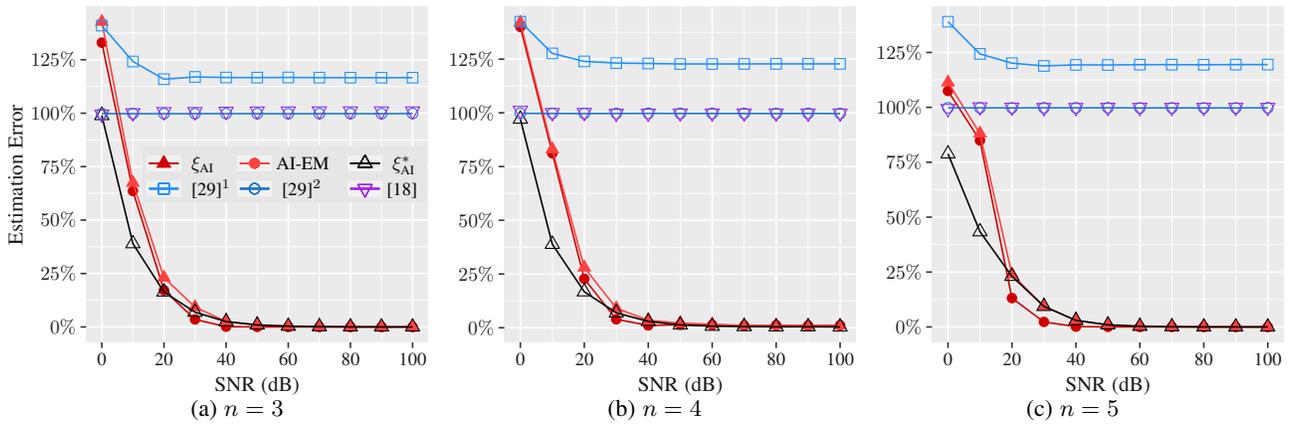

\centering
	\begin{subfigure}[h]{0.3\textwidth}
		\centering
		\hspace*{-0.6cm}
		\input{./figures/n=3.tex}
		\vspace*{-0.75cm}
		\caption{$n=3$} \label{fig:m500n3}
	\end{subfigure}
	\begin{subfigure}[h]{0.3\textwidth}
		\centering
		\hspace*{-0.2cm}
		\input{./figures/n=4.tex}
		\vspace*{-0.75cm}
		\caption{$n=4$} \label{fig:m500n4}
	\end{subfigure}
	\begin{subfigure}[h]{0.3\textwidth}
	\centering
	\input{./figures/n=5.tex}
	\vspace*{-0.75cm}
	\caption{$n=5$} \label{fig:m500n5}
	\end{subfigure}
	\caption{Estimation error for fully shuffled data, for different values of $n$ and $\text{SNR}$, with $m=500$ fixed. Average over $100$ trials. }
\end{figure*}

Fig. \ref{fig:partial10} depicts the estimation error for different 
percentages $0\% : 10 \% : 100\%$ of partially shuffled data, for $n=4, \, m=500$, and $\text{SNR}=40\text{dB}$ fixed. In such a case, only a subset of the entries of the vector $A \xi^*$ is shuffled according to a random permutation. As seen,   \cite{Slawski:EJS19}, \cite{Abid:arXiv18}$^1$ and \cite{Abid:arXiv18}$^2$ perform much better than for fully shuffled data. In fact, \cite{Slawski:EJS19}, \cite{Abid:arXiv18}$^1$ are comparable to AI-EM for up to $40\%-50\%$ shuffled data, upon which percentage they start deteriorating and eventually break down. On the other hand, \cite{Abid:arXiv18}$^2$ breaks down as soon as $20\%$ of the data have been shuffled. A more detailed behavior is also shown for $0\%:1\%:10\%$ shuffled data in the same setting. We see that \cite{Slawski:EJS19}, \cite{Abid:arXiv18}$^1$ and AI-EM are almost perfectly accurate, while \cite{Abid:arXiv18}$^2$ presents already an error of more than $10\%$ for $4\%$ shuffled data, and more than $60\%$ for $10\%$ shuffled data, i.e., \cite{Abid:arXiv18}$^2$ seems to be accurate only when the percentage of shuffled data is very small.

\begin{figure} 
\centering
	\begin{subfigure}[h]{0.3\textwidth}
		\centering
		\input{./figures/n=4_partial_100_SNR.tex}
		%\caption{partial}
	\end{subfigure}	
	%\caption{Estimation error for partially shuffled data of varying percentages, and $m=500, \, n=4, \, \text{SNR}=40\text{dB}$ fixed ($100$ trials).} \label{fig:partial}
\end{figure}
\begin{figure} 
\centering
	\begin{subfigure}[h]{0.3\textwidth}
		\centering
		\input{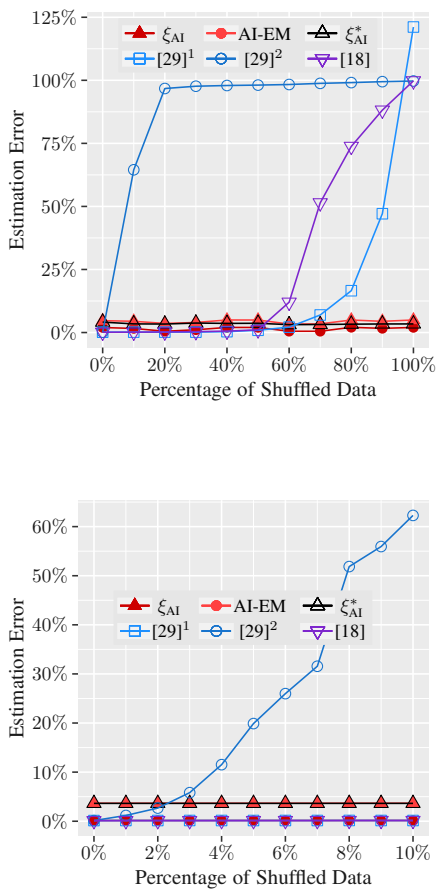}
		%\caption{partial}
	\end{subfigure}	
	\caption{Estimation error for partially shuffled data of varying percentages $0\%:10\%:100\%$ (top) and $0\%:1\%:10\%$ (bottom) for $m=500, \, n=4, \, \text{SNR}=40\text{dB}$ fixed ($100$ trials).} \label{fig:partial10}
\end{figure}

Table \ref{table:n2to6} provides numerical data for the estimation error and running times for different values of $n=3,4,5,6$. 
We see that for $n=3,4$ solving the polynomial system $\tilde{\P}$ via our custom solvers requires only a few milliseconds. On other hand, the running time increases exponentially with $45$ seconds required to solve the $n=5$ system and about $37$ minutes for  $n=6$. As expected, the running time of the Alternating Minimization (AM) remains practically unaffected by the values of $n$ that we are considering. Interestingly, the estimation error of the AI-EM algorithm is also stable and of the order of $0.1\%$ regardless of $n$. 

\begin{table}[h!] 
\centering
	\ra{1.3}
	\caption{Running time (in milliseconds) of polynomial solvers and alternating minimization, and estimation errors for varying $n$. $\text{SNR}=40\text{dB}, \, m=500$, fully shuffled data, $10$ trials.} \label{table:n2to6}
	\begin{tabular}{@{}rrccccc@{}}\toprule
		& \multicolumn{2}{c}{Running Time} & \phantom{c} & \multicolumn{3}{c}{Estimation Error} \\
		\cmidrule{2-3} \cmidrule{5-7} 
		$n$ & $\tilde{\mathcal{P}}$ & AM && $\xi_{\text{AI}}^*$ & $\xi_{\text{AI}}$ & AI-EM \\ \midrule
		$3$ & $0.7$  & $6$  && $2.3\%$ & $2.3\%$ & $0.1\%$\\
		$4$ & $9$  & $6$  && $3.9\%$ & $3.9\%$ & $0.1\%$\\
		$5$ & $45,157$ & $7$  && $1.4\%$ & $1.4\%$ & $0.1\%$\\
		$6$ & $2,243,952$ & $7$ && $1.2\%$ &$1.2\%$& $0.1\%$\\
		\bottomrule
	\end{tabular}
\end{table}

Table \ref{table:m} attests to the scalability of AI-EM in terms of $m$. Fixing $n=4$ and $\sigma=0.01$, the table reports running times and estimation errors for different values of $m$, ranging from $m=10^3$ to $m=10^5$. Indeed, solving the polynomial system requires $10$ msec for $m=10^3$ and only $268$ msec for $m=10^5$. The increase in the running time of the alternating minimization from $15$ msec to $1.3$ sec is more noticeable due to the complexity $\mathcal{O}(m \log m)$ of the sorting operation required to compute the optimal permutation at each iteration. For $m=1000$ the estimation error of AI-EM is only $0.4\%$, but as $m$ increases, the criterion \eqref{eq:xi-AI} for selecting the best root of $\tilde{\mathcal{P}}$ becomes less accurate; e.g., for $m=10^5$, $\xi_{\text{AI}}^*$ is $4.1\%$ away from $\xi^*$, as opposed to $6.8\%$ for $\xi_{\text{AI}}$.

\begin{table}[h!]\centering
	\ra{1.3}
	\caption{Run times (msec) and estimation errors for different $m$ with $\text{SNR}=40\text{dB}, \, n=4$ fixed. Fully shuffled data, $100$ trials.} \label{table:m}
	
	\begin{tabular}{@{}rrrcccc@{}}\toprule
		& \multicolumn{2}{c}{Running Time} & \phantom{c} & \multicolumn{3}{c}{Estimation Error} \\
		\cmidrule{2-3} \cmidrule{5-7} 
		$m$ & $\tilde{\mathcal{P}}$ & AM && $\xi_{\text{AI}}^*$ & $\xi_{\text{AI}}$ & AI-EM \\ \midrule
		$1,000$ &$10$& $15$  && $3.2\%$ & $3.3\%$ & $0.4\%$\\
		$5,000$ &$21$& $105$  && $3.5\%$ & $4.8\%$ & $2.5\%$\\
		$10,000$ &$32$& $281$ && $4.1\%$ &$5.9\%$ & $4.3\%$ \\
		$25,000$ &$66$&$357$ && $3.4\%$ & $5.2\%$ & $4.4\%$\\
		$50,000$ &$126$&$613$ && $3.7\%$ & $6.0\%$ & $5.5\%$\\
		$100,000$ &$268$&$1,271$ && $4.1\%$ & $6.8\%$ & $6.5\%$\\
		\bottomrule  
	\end{tabular}
\end{table} 

%%%%%%
\section{Conclusions}
We established the theory of an algebraic geometric solution to the problem of linear regression without correspondences. The main object of interest was a polynomial system $\mathcal{P}$ of $n$ equations in $n$ unknowns constraining the parameters $x\in \Re^n$ of the linear regression model. The main theoretical contribution was establishing that $\mathcal{P}$ is almost always consistent even for corrupted data with at most $n!$ complex roots. Our algorithmic proposal consisted of solving $\mathcal{P}$ and using its most appropriate root as initialization to the Expectation-Maximization algorithm. This yielded an efficient solution for small values of $n$, stable under mild levels of noise and outperforming existing state-of-the-art methods. 

%%%%
\appendices

%[The notion of dimension in algebraic geometry] \label{subsection:Dimension}

%%%
\setcounter{prp}{0}
\setcounter{ex}{0}
\setcounter{thm}{0}
\setcounter{dfn}{0}
\renewcommand{\theprp}{\Alph{section}\arabic{prp}}
\renewcommand{\theex}{\Alph{section}\arabic{ex}}
\renewcommand{\thelem}{\Alph{section}\arabic{lem}}
\renewcommand{\thethm}{\Alph{section}\arabic{thm}}
\renewcommand{\thedfn}{\Alph{section}\arabic{dfn}}

\section{Geometric characterization of dimension} \label{appendix:geometric-dimension}
Let $f_1,\dots,f_s$ be polynomials in $\Re[x]=\Re[x_1,\dots,x_n]$ and consider their common root locus $\V_{\Re^n}(f_1,\dots,f_s)$, called an \emph{algebraic variety}, defined as
\begin{align}
\V_{\Re^n}(f_1,\dots,f_s) := \{\xi \in \Re^n: \, f_k(\xi)=0, \, \forall k \in [s]\}.
\end{align} 
What is the dimension of $\V_{\Re^n}(f_1,\dots,f_s)$? If $s=1$, then we have a single equation and one intuitively expects $\V_{\Re^n}(f_1)$ to be a hypersurface of $\Re^n$ having dimension $n-1$; this is in analogy with linear algebra, where a single linear equation defines a linear subspace of dimension one less than the ambient dimension. However, as Example \ref{ex:InfiniteComplex} shows, it may be the case that $\V_{\Re^n}(f_1)$ consists of a single point (or even no points at all), in which case $\dim \V_{\Re^n}(f_1)$ should be zero (or $-1$ if the variety is empty), in analogy with linear algebra where a linear subspace has zero dimension only if it contains a single point (the origin $0$). 

To resolve the above issue and have a consistent definition of dimension that generalizes the linear algebraic one, it is necessary that we consider the common root locus of the polynomials in the \emph{algebraic closure} $\Ce$ of $\Re$:
\begin{align}
\V_{\Ce^n}(f_1,\dots,f_s) := \{\xi \in \Ce^n: \, f_k(\xi)=0, \, \forall k \in [s]\}.
\end{align} In that case, there is a well developed theory (\cite{Hartshorne-1977,Matsumura-2006,Eisenbud-2004}) that leads to consistent characterizations of the dimension of the geometric object $\V_{\Ce^n}(f_1,\dots,f_s) \subset \Ce^n$ and that of its algebraic counterpart $\{f_1,\dots,f_s\} \subset \Ce[x]$. The next definition provides the geometric characterization of $\dim \V_{\Ce^n}(f_1,\dots,f_s)$.

\begin{dfn} \label{dfn:DimensionGeometric}
Defining $\Y \subset \Ce^n$ to be closed if it is of the form $\Y=\V(g_1,\dots,g_r)$ for some polynomials $g_1,\dots,g_r \in \Ce[x]$, and irreducible if it is not the union of two proper closed subsets, $\dim \V_{\Ce^n}(f_1,\dots,f_s)$ is defined to be the largest\footnote{The acute reader may notice that there is a-priori no guarantee that such a maximal integer exists. However, this is true because $\Ce^n$ is a \emph{Noetherian topological space}, a technical notion that is beyond the scope of this paper.} non-negative integer $d$ such that there exists a chain of the form 
\begin{align}
\V_{\Ce^n}(f_1,\dots,f_s) \supset \Y_0 \supsetneq \Y_1 \supsetneq \Y_2 \supsetneq \cdots \supsetneq \Y_d,
\end{align} where each $\Y_i$ is a closed irreducible subset of $\V_{\Ce^n}(f_1,\dots,f_s)$.
\end{dfn} 

\noindent Definition \ref{dfn:DimensionGeometric} is a generalization of the notion of dimension in linear algebra: if $\Y$ is a linear subspace of $\Ce^n$, then $\dim \Y$ is precisely equal to the maximal length of a descending chain of linear subspaces that starts with $\Y$; one can get such a chain by removing a single basis vector of $\Y$ at each step\footnote{For more information on the algebraic geometric structure of linear subspaces the reader is referred to Appendix C in \cite{Tsakiris:SIAM17}.}. 

\begin{ex} \label{ex:VectorSpaceChain}
With $e_i$ the vector with zeros everywhere except a $1$ at position $i$, and $\Y_i=\Span(e_1,\dots,e_{n-i})$, $\Ce^n$ admits a chain
\begin{align}
\Ce^n=\Y_0 \supsetneq \Y_1 \supsetneq \Y_2 \supsetneq \cdots \supsetneq \Y_{n-1} \supsetneq \Y_n:=\{0\}.
\end{align}
\end{ex}

\noindent A very important structural fact about algebraic varieties is the following decomposition theorem.
\begin{prp} \label{prp:IrreducibleDecomposition}
Let $\Y = \V_{\Ce^n}(f_1,\dots,f_s)$ for some $f_i \in \Ce[x]$. Then $\Y$ can be written uniquely as 
$\Y = \Y_1 \cup \dots \cup \Y_{\ell}$, for some positive integer $\ell$, where the $\Y_i$ are irreducible closed sets of $\Ce^n$ (see Definition \ref{dfn:DimensionGeometric}), and they are minimal, in the sense that if one removes one of the $\Y_i$, the resulting union is a strictly smaller set than $\Y$. The $\Y_i$ are called the irreducible components of $\Y$. 
\end{prp}
 
\noindent Definition \ref{dfn:DimensionGeometric} together with Proposition \ref{prp:IrreducibleDecomposition} ensure that the only algebraic varieties $\V_{\Ce^n}(f_1,\dots,f_s)$ that have dimension zero are the ones that consist of a finite number of points; these are precisely the varieties of interest in this paper.

\begin{prp} \label{prp:ZeroDimensionFinitePoints}
Let $\Y = \V_{\Ce^n}(f_1,\dots,f_s)$. Then $\dim \Y = 0$ if and only if $\Y$ consists of a finite number of points of $\Ce^n$.
\end{prp}

%%%
\setcounter{prp}{0}
\setcounter{ex}{0}
\setcounter{thm}{0}
\setcounter{dfn}{0}
\renewcommand{\theprp}{\Alph{section}\arabic{prp}}
\renewcommand{\theex}{\Alph{section}\arabic{ex}}
\renewcommand{\thelem}{\Alph{section}\arabic{lem}}
\renewcommand{\thethm}{\Alph{section}\arabic{thm}}
\renewcommand{\thedfn}{\Alph{section}\arabic{dfn}}

\section{Algebraic characterization of dimension}
Even though Definition \ref{dfn:DimensionGeometric} is quite intuitive, it is not as convenient to use in practice, since one is usually given polynomials $f_1,\dots,f_s$ and wants to determine whether $\V_{\Ce^n}(f_1,\dots,f_s)$ has zero dimension, without having to solve the polynomial system. This is the case in this paper, where, e.g., to prove Theorem \ref{thm:Main} we need to show that $\V_{\Ce^n}(\hat{p}_1,\dots,\hat{p}_n)$ has zero dimension for any suitable $A,y$: clearly, computing the common root locus of $\hat{p}_1,\dots,\hat{p}_n,$ as a function of $A, y$, is extremely challenging if not impossible (except for $n=1,2$). This is precisely where the algebraic characterization of $\dim \V_{\Ce^n}(f_1,\dots,f_s)$ comes in handy, since it allows its computation solely from the algebraic structure of $f_1,\dots,f_s$. 

To introduce this algebraic notion of dimension we first need the notion of an \emph{ideal} of $\Ce[x]$. Given polynomials $f_1,\dots,f_s \in \Ce[x]$, the ideal generated by these polynomials, denoted by $(f_1,\dots,f_s)$, is the set of all linear combinations of the $f_i$, but in contrast to linear algebra, the coefficients of the linear combination are allowed to be polynomials themselves:
\vspace{-0.09in}
\begin{align}
(f_1,\dots,f_s) := \Bigg\{\sum_{i=1}^ s g_i f_i, \, \forall g_i \in \Ce[x]\Bigg\}.
\end{align} Next we need the notion of a \emph{prime ideal}. An ideal $\P \subsetneq \Ce[x]$ is called prime if it satisfies the following property: whenever the product of two polynomials is inside $\P$, then at least one of these polynomials must be inside $\P$. With that we have:

\begin{prp} \label{prp:DimensionAlgebraic}
$\dim \V_{\Ce^n}(f_1,\dots,f_s)$ is the largest non-negative integer $d$ such that there exists a chain of the form 
\begin{align}
\P_0 \supsetneq \P_1 \supsetneq \P_2 \supsetneq \cdots \supsetneq \P_d \supset (f_1,\dots,f_s),
\end{align} where each $\P_i$ is a prime ideal of $\Ce[x]$.
\end{prp}

\begin{ex} \label{ex:VectorSpacePrimeChain}
Continuing with Example \ref{ex:VectorSpaceChain}, $\Ce^n = \V_{\Ce^n}(0)$ and
\begin{align}
\Y_i:=\Span(e_1,\dots,e_{n-i})=\V_{\Ce^n}(x_{n-i+1},\dots,x_n).
\end{align} Since every ideal of the form $(x_{n-i+1},\dots,x_n)$ is prime\footnote{For a justification, the reader is referred to Proposition 53 in \cite{Tsakiris:SIAM17}.}, we have the ascending chain of prime ideals of length $n$:
\begin{align}
(x_1,\dots,x_n) \supsetneq (x_2,\dots,x_n)  \supsetneq \cdots \supsetneq (x_n) \supsetneq(0).
\end{align}
\end{ex}

\noindent We close this section by noting that the main tool behind the proof of Proposition \ref{prp:DimensionAlgebraic} is the famous \emph{Hilbert's Nullstellensatz}, stated next, which holds over $\Ce$ but not over $\Re$. This is why we need to work over $\Ce$ to get consistent geometric and algebraic characterizations of the dimension of an algebraic variety.

\begin{prp} \label{prp:Nullstellensatz}
Let $f_1,\dots,f_s$ be polynomials of $\Ce[x]$. Then
\begin{itemize} \item $\V_{\Ce^n}(f_1,\dots,f_s) = \emptyset$ if and only if $1 \in (f_1,\dots,f_s)$.
\item Suppose that $\V_{\Ce^n}(f_1,\dots,f_s) \neq \emptyset$ and let $f$ be a polynomial such that $f(\xi)=0, \, \forall \xi \in \V_{\Ce^n}(f_1,\dots,f_s)$. Then $f^{\ell} \in (f_1,\dots,f_s)$ for some positive integer $\ell$.
\end{itemize}
\end{prp}

%%%
\setcounter{prp}{0}
\setcounter{ex}{0}
\setcounter{thm}{0}
\setcounter{dfn}{0}
\renewcommand{\theprp}{\Alph{section}\arabic{prp}}
\renewcommand{\theex}{\Alph{section}\arabic{ex}}
\renewcommand{\thelem}{\Alph{section}\arabic{lem}}
\renewcommand{\thethm}{\Alph{section}\arabic{thm}}
\renewcommand{\thedfn}{\Alph{section}\arabic{dfn}}

\section{Dimension and homogenization} \label{subsubsection:Homogenization}
A monomial of degree $d$ is a polynomial of the form $x^{\alpha}:= x_1^{\alpha_1} \cdots x_n^{\alpha_n}$, where $\alpha =[\alpha_1,\dots,\alpha_n]^\top$ is a vector of non-negative integers such that $\alpha_1+\cdots+\alpha_n=d$. Every polynomial $f$ can be uniquely written as a linear combination of monomials
\begin{align}
f= \sum_{\alpha \in \mathcal{A}} c_{\alpha} x^{\alpha}  \in \Ce[x], \label{eq:f-Monomials}
\end{align} where $\mathcal{A}$ is a finite set of \emph{multi-exponents} and $c_{\alpha} \in \Ce$ are the corresponding coefficients. Then $f$ is called \emph{homogeneous} of degree $d$ if all its monomials have the same degree $d$.  

Let $f_1,\dots,f_s$ be a set of polynomials of $\Ce[x]$. For rather subtle reasons beyond the scope of this paper, further characterizing the dimension of $\V_{\Ce^n}(f_1,\dots,f_s)$ beyond Proposition \ref{prp:DimensionAlgebraic} is simpler when all the $f_i$ are homogeneous (this will be discussed in the next section). When this is not the case, there is a simple procedure called \emph{homogenization}, through which we can convert non-homogenous polynomials to homogeneous.   

Suppose that the polynomial $f$ given in \eqref{eq:f-Monomials} is not homogeneous. Let $d$ be the maximal degree among all the monomials of $f$, i.e., $d = \max\{\sum_{i \in [n]} \alpha_i:  \, \alpha \in \mathcal{A}\}$. Then the homogenization of $f$ is a polynomial in the extended polynomial ring $\Ce[x,t]$ with one additional variable $t$, defined as 
\begin{align}
f^{(h)} := \sum_{\alpha \in \mathcal{A}} c_{\alpha} x^{\alpha} t^{d-\sum_{i \in [n]} \alpha_i}  \in \Ce[x,t].
\end{align}

\begin{ex}
A non-homogeneous polynomial $f$ of degree $6$ and its homogenization $f^{(h)}$:
\begin{align}
f &= x_1^3x_2^2x_3+x_2^2 +x_3 \in \Ce[x_1,x_2,x_3],\\
f^{(h)} &= x_1^3x_2^2x_3+x_2^2t^4 +x_3t^5 \in \Ce[x_1,x_2,x_3,t].
\end{align} 
\end{ex}

Now, let $\I$ be the ideal generated by some polynomials $f_1,\dots,f_s \in \Ce[x]$ and consider the homogenization of this ideal
\begin{align}
    \I^{(h)} = \{f^{(h)}: \, f \in \I \} \subset \Ce[x,t]. 
\end{align} We note here the subtle fact that $\I^{(h)}$ certainly contains 
 $f_1^{(h)},\dots,f_s^{(h)}$, but in principle it is larger than the ideal generated by $f_1^{(h)},\dots,f_s^{(h)}$, as the next example illustrates. 
 
 \begin{ex}
 Let $f_1=x_1^2+x_2,\, f_2=x_1^2+x_3$ be polynomials of $\Ce[x_1,x_2,x_3]$. Then 
 $f_1^{(h)}=x_1^2+x_2t,\, f_2^{(h)}=x_1^2+x_3t$.
 Now, the polynomial $x_2-x_3=f_1-f_2$ is in the ideal $\I=(f_1,f_2)$, and it is already homogeneous so that $x_2-x_3 \in \I^{(h)}$. However $x_2-x_3$ is not inside the ideal $(f_1^{(h)},f_2^{(h)})=(x_1^2+x_2t,x_1^2+x_3t)$, since the latter only contains elements of degree $2$ and higher.  
 \end{ex}
 
 Since the elements of $\I^{(h)}$ are polynomials in $n+1$ variables, they define an algebraic variety\footnote{Let $\J \subset \Ce[x]$ be an ideal. Then \emph{Hilbert's Basis Theorem} guarantees that $\J$ always has a finite set of generators, i.e., there is a positive integer $\ell$ and polynomials $g_1,\dots,g_{\ell} \in \Ce[x]$ such that $\J = (g_1,\dots,g_{\ell})$.}  $\Y^{(h)}=\V_{\Ce^{n+1}}\Big(\I^{(h)}\Big)$ of $\Ce^{n+1}$. What is the relationship between $\Y$ and $\Y^{(h)}$? It is actually not hard to see that if $[\xi_1,\dots,\xi_n]^\top$ is a point of $\Y$, then $\lambda [\xi_1,\dots,\xi_n,1]^\top$ is a point of $\Y^{(h)}$, for any $\lambda \in \Ce$. Hence any non-zero point $\xi$ of $\Y$ gives rise to an entire line inside $\Y^{(h)}$; this line passes through the origin and $\xi$, and its intersection with the hyperplane $t=1$ can be used to recover the original point $\xi$. Hence $\Y^{(h)}$ is called the \emph{affine cone} over $\Y$ with vertex $0 \in \Ce^{n+1}$. Moreover, the variety $\Y \subset \Ce^n$ is embedded inside the affine cone through a mapping that takes points to lines. In addition, $\Y^{(h)}$ contains so-called \emph{points at infinity}, which are obtained by setting $t=0$. As it turns out, there is a tight topological relationship between $\Y$ and $\Y^{(h)}$ and the important fact for our analysis is the following dimension theorem; see \cite{Tsakiris:AffinePAMI17} for a detailed discussion for the non-expert reader in the context of subspace clustering.

\begin{prp} \label{prp:AffineConeDimension}
Let $\Y$ be an algebraic variety of $\Ce^n$ and let $\Y^{(h)} \subset \Ce^{n+1}$ be its affine cone. Then $\dim \Y = \dim \Y^{(h)}-1$.
\end{prp}

\begin{ex}
Let $\Y$ be an affine line of $\Ce^2$ given by the equation $\alpha x_1 + \beta x_2 +\gamma=0$. Then $\Y^{(h)}$ is a plane through the origin in $\Ce^3$ given by the equation $\alpha x_1 + \beta x_2 +\gamma t=0$. 
\end{ex}

\noindent The next fact, known as \emph{Bezout's Theorem}, will be used in bounding the number of points of the zero-dimensional variety of Theorem \ref{thm:Main}. 
\begin{prp} \label{prp:Bezout}
Let $h_1,\dots,h_n$ be homogeneous polynomials of $\Ce[x,t]$ of degrees $\deg(h_i) = d_i, \, i \in [n]$. If $\V_{\Ce^{n+1}}(h_1,\dots,h_n)$ is a finite union of lines through the origin, then the number of these lines is at most $d_1 d_2 \cdots d_n$.
 \end{prp}

%%%
\setcounter{prp}{0}
\setcounter{ex}{0}
\setcounter{thm}{0}
\setcounter{dfn}{0}
\renewcommand{\theprp}{\Alph{section}\arabic{prp}}
\renewcommand{\theex}{\Alph{section}\arabic{ex}}
\renewcommand{\thelem}{\Alph{section}\arabic{lem}}
\renewcommand{\thethm}{\Alph{section}\arabic{thm}}
\renewcommand{\thedfn}{\Alph{section}\arabic{dfn}}

\section{Regular sequences} \label{appendix:regular-sequences}
In \S \ref{subsection:SymmetricPolynomials} we argued that if the polynomials $\hat{p}_1,\dots,\hat{p}_n \in \Ce[x]$ are \emph{sufficiently independent}, then $\dim \V_{\Ce^n}(\hat{p}_1,\dots,\hat{p}_n) = 0$, i.e., the dimension of the algebraic variety drops precisely by the number $n$ of its defining equations. More generally, the precise notion of what \emph{sufficiently independent} should mean for polynomials $f_1,\dots, f_s, \, s \le n$, so that $\dim \V_{\Ce^n}(f_1,\dots,f_s)=n-s$, is easier to characterize when all the $f_i$ are homogeneous. The right notion is that of a \emph{regular sequence}.

\begin{dfn} \label{dfn:regular-sequence}
Let $f_1,\dots,f_s$ be polynomials of $\Ce[x]$. Then $f_1,\dots,f_s$ is a regular sequence if $(f_1,\dots,f_s) \subsetneq \Ce[x]$, and for every $i=2,\dots,s$ the following property is true: whenever there is a polynomial $g$ such that 
$f_i g \in (f_1,\dots,f_{i-1})$, then we must have $g \in (f_1,\dots,f_{i-1})$.
\end{dfn} \noindent The crucial fact for our analysis is the following.

\begin{prp} \label{prp:RegularSequences}
 Let $f_1,\dots,f_s, \, s \le n$, be non-constant homogeneous polynomials of $\Ce[x]$. Then $\dim \V_{\Ce^n}(f_1,\dots,f_s) = n-s$, if and only if $f_1,\dots,f_s$ is a regular sequence.
\end{prp} Given a regular sequence of  polynomials $f_1,\dots,f_s$ in $\Ce[x]$ of length $s<n$, it is of interest to be able to augment this sequence to a regular sequence $f_1,\dots,f_s,g$ of length $s+1$. The simplest type of a homogeneous polynomial $g$ that one may consider is a linear form $g = \ell^\top x$, which represents a hyperplane with normal vector $\ell \in \Ce^n$. As it turns out, almost all such hyperplanes qualify, with the exception of those with normal vector $\ell$ that lies inside an algebraic variety of $\Ce^n$ determined by $f_1,\dots,f_s$. 

\begin{prp} \label{prp:GenericHyperplanes}
Let $f_1,\dots,f_s, \, s < n$, be a regular sequence of homogeneous polynomials of $\Ce[x]$. If $\ell \in \Ce^n$ is a generic vector, then $f_1,\dots,f_s,\ell^\top x$ is a regular sequence.
\end{prp}

%%%%
\setcounter{prp}{0}
\setcounter{ex}{0}
\setcounter{thm}{0}
\setcounter{dfn}{0}
\renewcommand{\theprp}{\Alph{section}\arabic{prp}}
\renewcommand{\theex}{\Alph{section}\arabic{ex}}
\renewcommand{\thelem}{\Alph{section}\arabic{lem}}
\renewcommand{\thethm}{\Alph{section}\arabic{thm}}
\renewcommand{\thedfn}{\Alph{section}\arabic{dfn}}

\section{Initial ideals} \label{appendix:initial-ideals}

The notion of the \emph{initial ideal} $\In_{<}(\I)$ of an ideal $\I \subset \Ce[x]$ with respect to a \emph{monomial order} $<$ is a central one in computational algebraic geometry \cite{Cox:2007}. A more advanced object that is needed for our analysis in this paper is the initial ideal $\In_{w}(\I)$ of $\I$ with respect to a \emph{weight-order} \cite{BrunsConca2003, HerzogHibi2011}, which we introduce next.

Let $w=[w_1,\dots,w_n]^\top$ be a vector of positive integers. To each variable $x_i$ of $\Ce[x]$ we assign the \emph{weight} $w_i$, and to each monomial $x^{\alpha}=x_1^{\alpha_1} \cdots x_n^{\alpha_n}$ the \emph{weighted degree} $d_{w}(x^{\alpha}):=w_1 \alpha_1 + \cdots + w_n \alpha_n$. A polynomial $f$ is called $w$-homogeneous, if all its monomials have the same weighted degree. 
\begin{ex} \label{ex:w-homogeneous}
Let $w = [1,2,3]^\top$ and let $f=x_1x_2 + x_3$. Then $f$ is not homogeneous in the usual sense (see Appendix \ref{subsubsection:Homogenization}), but it is $w$-homogeneous of degree $3$.
\end{ex}
Now let $f \in \Ce[x]$ be any polynomial. Then $f$ can be uniquely written as $f = f^{(d_1)}+f^{(d_2)}+\cdots+f^{(d_{s})}$, with $d_1>d_2>\cdots>d_{s}>0$, where each $f^{(d_i)}$ is a $w$-homogeneous polynomial of degree $d_i$. We define the \emph{initial form} of $f$ with respect to $w$ as $\In_{w}(f):=f^{(d_1)}$. Given an ideal $\I=(f_1,\dots,f_s)$, we define $\In_{w}(I)$ to be the ideal generated by all initial forms $\In_{w}(f)$ for all $f \in \I$. That is, $h \in \In_{w}(I)$ if and only if there exist polynomials $g_i \in \Ce[x], \, i \in [s]$, such that $h =  {\In_{w}}(g_1 f_1+\cdots+g_s f_s)$.
\begin{ex}
Let $f=x_1x_2 + x_3+x_1^2+x_2+x_1$ and $w$ as in Example \ref{ex:w-homogeneous}. Then $f=f^{(3)}+f^{(2)}+f^{(1)}$ with
\begin{align}
f^{(3)} = x_1x_2 + x_3, \, f^{(2)}=x_1^2+x_2, \, f^{(1)}= x_1.
\end{align} Moreover, $\In_{w}(f) = x_1x_2 + x_3$.
\end{ex} The initial ideal $\In_{w}(I)$ is certainly a significantly simpler object than the ideal $\I$ itself, since it retains only the initial information about $\I$, so to speak. What is remarkable though, is that many structural properties of $\I$ are inherited from those of $\In_{w}(I)$. For this paper, the most important relationship is that the varieties defined by these two ideals have the same dimension:

\begin{prp} \label{prp:InitialIdeal-Dimension}
Let $\I \subset \Ce[x]$ be an ideal, $w \in \mathbb{Z}^n_{>0}$ a weight, and $\In_{w}(\I) \subset \Ce[x]$ the initial ideal of $\I$ with respect to $w$. Then 
\begin{align}
\dim \V_{\Ce^n}(\I) = \dim \V_{\Ce^n}({\In\nolimits_{w}}(\I)).
\end{align}
\end{prp}

\noindent  Hence to compute the dimension of the algebraic variety defined by an ideal $\I=(f_1,\dots,f_s)$, we may instead use the simpler object $\In_{w}(\I)$. But how can we efficiently compute a set of generators for $\In_{w}(\I)$ given $f_1,\dots,f_s$? Note here that $(\In_{w}(f_1),\dots,\In_{w}(f_s)) \subset \In_{w}(\I)$ but equality does not hold in general, as the next example shows. 
\begin{ex}
Let $w = [1,2,3]^\top$ and $\I=(f_1,f_2) \subset \Ce[x]$ with
\begin{align}
f_1 = x_1^9 + x_3 + x_2, \, \, \, f_2 = x_1^9+x_2 x_3.
\end{align} Then $\In_{w}(f_1) = \In_{w}(f_2) = x_1^9$ and so $(\In_{w}(f_1),\In_{w}(f_2)) =(x_1^9)= \Big\{x_1^9 g: \, \forall g \in \Ce[x]\Big\}$. On the other hand $f_2-f_1 = x_2 x_3 - x_3 -x_2 \in \I$ and so $\In_{w}(f_2-f_1) = x_2x_3 \in \In_{w}(\I)$. But clearly, $x_2 x_3 \not\in (x_1^9)$. Hence $(\In_{w}(f_1),\In_{w}(f_2)) \subsetneq \In_{w}(\I)$. 
\end{ex}

\noindent The next fact, well known to experts in \emph{Gr\"obner basis} theory and otherwise much less known, describes a situation where we can directly get the generators of $\In_{w}(\I)$ from the generators of $\I$.

\begin{prp} \label{prp:InitialRegular}
Let $w \in \mathbb{Z}^n_{>0}$ be a weight, and $f_1,\dots,f_s$ be polynomials of $\Ce[x]$, such that 
$\In_{w}(f_1),\dots,\In_{w}(f_s)$ is a regular sequence. Then $\In_{w}(\I) = (\In_{w}(f_1),\dots,\In_{w}(f_s))$.
\end{prp}

%%%%%%%%%%%%%%%%%%%%%%%%%%%%%%%%%%%%%%%%%%%%%%%%%%
\bibliographystyle{IEEEtran}


%\bibliography{/Users/manolistsakiris/Dropbox/Papers/biblio/alias,/Users/manolistsakiris/Dropbox/Papers/biblio/vidal,/Users/manolistsakiris/Dropbox/Papers/biblio/vision,/Users/manolistsakiris/Dropbox/Papers/biblio/math-Manolis,/Users/manolistsakiris/Dropbox/Papers/biblio/learning,/Users/manolistsakiris/Dropbox/Papers/biblio/sparse,/Users/manolistsakiris/Dropbox/Papers/biblio/geometry,/Users/manolistsakiris/Dropbox/Papers/biblio/dti,/Users/manolistsakiris/Dropbox/Papers/biblio/recognition,/Users/manolistsakiris/Dropbox/Papers/biblio/surgery,/Users/manolistsakiris/Dropbox/Papers/biblio/coding,/Users/manolistsakiris/Dropbox/Papers/biblio/matrixcompletion,/Users/manolistsakiris/Dropbox/Papers/biblio/segmentation,/Users/manolistsakiris/Dropbox/Papers/biblio/dataset,/Users/manolistsakiris/Dropbox/Papers/biblio/ShuffledLinearRegression,/Users/manolistsakiris/Dropbox/Papers/biblio/manolis-learning,SLR-TIT-18,HS-ICML19,/Users/manolistsakiris/Dropbox/Papers/biblio/Tsakiris}

\end{document}